\documentclass{article}
\usepackage[dvips]{graphicx}
\usepackage{amsmath, amsfonts, amsthm, amssymb,epsfig}
\usepackage{graphicx}

\title{Learning the effect of latent variables in Gaussian Graphical models with unobserved variables}

\author{
Marina Vinyes \\
Universit\'{e} Paris-Est, LIGM (UMR8049)\\
Ecole des Ponts \\
Marne-la-Vall\'{e}e, France \\
\texttt{marina.vinyes@imagine.enpc.fr}
\and
Guillaume Obozinski  \\
Universit\'{e} Paris-Est, LIGM (UMR8049)\\
Ecole des Ponts \\
Marne-la-Vall\'{e}e, France \\
\texttt{guillaume.obozinski@enpc.fr} }

\usepackage{natbib}
\usepackage{color}
\usepackage{algorithm}
\usepackage{algpseudocode}
\usepackage{enumitem}
\usepackage{xr}
\usepackage{xcolor}
\usepackage{tikz}
\colorlet{BLUE}{blue}
\usepackage{cleveref}

\def\sbot{{\scriptscriptstyle \bot}}
\def\A{\mathcal{A}}
\def\Mi{M^{\scriptscriptstyle (I)}}
\def\Li{L^{\scriptscriptstyle (I)}}
\def\Lis{L^{\scriptscriptstyle (I)\,{\scriptstyle*}}}
\def\Qi{Q^{\scriptscriptstyle (I)}}

\def\Ni{N^{\scriptscriptstyle (I)}}

\def\tr{{\rm tr}}
\newcommand\itgset[1]{[\![\,#1\,]\!]}
\def\RR{\mathbb{R}}

\newcommand\GOD[1]{}

\def\st{\text{s.t.}}
\def\supp{\text{supp}}
\def\sign{\,\text{sign}}
\newcommand{\dotp}[2]{\langle #1,#2 \rangle}

\def\BIT{\begin{itemize}}
\def\EIT{\end{itemize}}
\def\BET{\begin{enumerate}}
\def\EET{\end{enumerate}}

\def\RR{\mathbb{R}}

\def\T{\mathcal{T}}
\def\P{\mathcal{P}}

\def\I{\mathcal{I}}

\def\op{{\rm op}}
\def\supp{{\rm Supp}}
\def\st{\text{s.t.}}

\def\xxi{\zeta}

\newtheorem{theorem}{Theorem}
\newtheorem{lemma}{Lemma}
\newtheorem{proposition}{Proposition}

\newtheorem{corollary}[theorem]{Corollary}
\newtheorem{definition}[theorem]{Definition}

\newtheorem{claim}{Claim}

\newtheorem{assumption}{Assumption}
\def\BIT{\begin{itemize}}
\def\EIT{\end{itemize}}
\def\BET{\begin{enumerate}}
\def\EET{\end{enumerate}}

\def\RR{\mathbb{R}}

\def\T{\mathcal{T}}
\def\I{\mathcal{I}}

\def\op{{\rm op}}
\def\tr{{\rm tr}}

\def\supp{{\rm Supp}}
\def\st{\text{s.t.}}

\def\tauu{{\overline{\tau}}}
\def\taul{{\underline{\tau}}}
\def\etau{{\overline{\eta}}}
\def\etal{{\underline{\eta}}}
\newcommand{\lambdasort}[1]{\lambda^{\scriptscriptstyle \!(#1)}}
\newcommand{\lambdatildesort}[1]{\tilde{\lambda}^{\scriptscriptstyle\!(#1)}}

\begin{document}
\maketitle  

\begin{abstract}
The edge structure of the graph defining an undirected graphical model describes precisely the structure of dependence between the variables in the graph.
In many applications, the dependence structure is unknown and it is desirable to learn it from data, often because it is a preliminary step to be able to ascertain causal effects. This problem, known as structure learning, is a hard problem in general, but for Gaussian graphical models it is slightly easier because the structure of the graph is given by the sparsity pattern of the precision matrix of the joint distribution, and because independence coincides with decorrelation.\\  

A major difficulty too often ignored in structure learning is the fact that if some variables are not observed, the marginal dependence graph over the observed variables will possibly be significantly more complex and no longer reflect the direct dependences that are potentially associated with causal effects. This is the problem of confounding variables. In this work, we consider a family of latent variable Gaussian graphical models (LVGGM) in which the graph of the joint distribution between observed and unobserved variables is sparse, and the unobserved variables are conditionally independent given the others. 
Prior work \citep{chandrasekaran2010} was able to recover the connectivity between observed variables, but could only identify the subspace spanned by unobserved variables, whereas we propose a convex optimization formulation based on structured matrix sparsity to estimate the complete connectivity of the original complete graph including unobserved variables, given the knowledge of the number of missing variables, and a priori knowledge of their level of connectivity. Our formulation is supported by a theoretical result of identifiability of the latent dependence structure for sparse graphs in the infinite data limit. We propose an algorithm leveraging recent active set methods, which performs well in the experiments  we ran on synthetic data.

\end{abstract}

\section{Introduction}
\label{intro}

Graphical models provide a sound theoretical framework to model a joint probability distribution with complex interdependences between a potentially large number of random variables, with applications in several fields including genomics and finance among others.\\  

In the Gaussian Graphical Models (GGM) literature, a central problem is to estimate the inverse covariance matrix, also known as the \emph{precision} or \emph{concentration matrix}. The sparsity pattern of the concentration matrix in Gaussian models corresponds to the structure of the graph; more precisely, the nonzeros of the concentration matrix correspond to the edges of the underlying undirected graphical model, which encode pairs of variables that are conditionally dependent given all the others. Identifying the structure of the graph is important since the number of parameters of the model grows linearly with the number of edges in the graph. 

The main  formulation for edge selection in the GGM setting is based on $\ell_1$-regularized maximum-likelihood \citep{friedman2008sparse,yuan2007model,banerjee2008model}, for which several algorithms have been proposed. The $\ell_1$ regularization provides convex formulation which induces the selection of some edges while implicitly removing others in the graph.

A serious practical difficulty is that applications in which all variables potentially relevant for the problem considered have been identified and measured are extremely rare. This entails the possible presence of \emph{confounding variables}. More precisely, some of the relevant variables may be latent and induce correlations between observed variables that can be misleading and can only be explained correctly if the presence of the latent variables that produce confounding effects is explicitly modeled. More precisely, when latent variables are missing, the marginalized precision matrix may not be sparse even if the full precision matrix is sparse. Imposing sparsity on the complete model results in a marginal precision matrix of the Latent Variable Gaussian Graphical Model (LVGGM) that has a sparse plus low-rank structure. \citet{chandrasekaran2010} consider a regularized maximum likelihood approach, using the $\ell_1$-norm to recover the sparse component and the trace norm to recover the low-rank component and show that they consistently estimate the sparsity pattern of the sparse component and the number of latent variables. Their method identifies the low-rank structure corresponding to the effect of latent variables but, in general,  it does not allow us to identify the covariance structure of each latent variable individually, or which observed variables are directly dependent on which unobserved ones.\\

In this work, we propose to impose more structure on the low rank matrix using a variant of the norms introduced in~\citet{richard2014tight} as a regularizer. This leads to formulations which yields estimates of the structure of the complete graphical model, and, in particular, make it possible to identify which observed variables are affected by which latent variables.\\

The paper is structured as follows: In Section \ref{related} we review the relevant prior literature. In Section \ref{sec:ggm}, we formulate the LVGGM estimation problem as a regularized convex problem that imposes a sparsity structure on the latent variables. In Section~\ref{subsec:alg}, we propose a convex formulation with a quadratic loss function, and an algorithm to solve this problem efficiently. In Section~\ref{sec:id}, we show that different parts of the complete graph are identifiable by our convex formulation, under appropriate conditions. We finally present experimental results in Section \ref{experiments}.

\section{Related Work}
\label{related}

To construct an interpretable graph in high-dimensional regimes, many authors have proposed applying
an $\ell_1$ penalty to the parameter associated with each edge, in order to encourage sparsity. For instance such an approach is taken by \citet{yuan2007model} and \citet{banerjee2008model} in the context of Gaussian graphical models. The first works to explore $\ell_1$ regularization in undirected graphical models over discrete variables are \citet{lee2007efficient,ravikumar2009high} and \citet{dahinden2007penalized}. In another line of work, authors have  considered $\ell_1$-regularization for learning structure in directed acyclic graphs given an ordering of the variables \citep{huang2006covariance,li2005using,levina2008sparse} and  \citet{schmidt2007learning,champion2018inferring} propose methods without assuming known ordering.\\

A conditional independence graph is sometimes expected to have particular structure. In the context of graphs with hub nodes, that is nodes with many neighbors, \citet{tan2014learning} present a convex formulation that involves a row-column overlap norm penalty. \citet{defazio2012convex} use a convex penalty adapted for a scale-free network in which the degree of connectivity of the nodes follows a power law distribution. \citet{tao2017inverse} impose an overlapping group structure on the concentration matrix.\\

Another useful problem, that is the focus of this paper, is finding the structure of Gaussian graphical models with unobserved variables.  \citet{chandrasekaran2010} introduced a convex formulation to find the number of  latent components and learn the structure of on the entire collection of variables. \citet{meng2014learning} also studied regularized
maximum likelihood estimation and derive Frobenius norm error bounds in the highdimensional
setting based on the restricted strong convexity. In order to speed
up the estimation of the sparse plus low-rank components, \citet{xu2017speeding} propose a sparsity constrained maximum likelihood estimator based on matrix factorization, and an efficient alternating proximal gradient descent algorithm with hard thresholding to solve it.  \citet{hosseini2016learning}  present a bi-convex formulation to jointly learn both a network among observed variables and densely connected and overlapping groups of variables, revealing the existence of potential latent variables. These methods identify the low-rank structure corresponding to the effect of latent variables but it does not allow us to identify the structure of the full model. In this work, we propose to impose more structure on the low rank matrix in order to obtain a  decomposition that gives the structure of the complete graphical model.

\subsection*{Notations}
$\itgset{p}$ denotes the set $\{1,...,p\}$ and $\mathcal{G}^p_k$ denotes the
set of subsets of $k$ elements in $\itgset{p}$. $|I|$ denotes the cardinality of a set $I$. If $v\in\RR^{p}$ is a vector, $\supp(v)$ denotes its support. If $M\in\RR^{p\times p}$ is a matrix, $I\subset\itgset{n}$ , $M_{II}\in\RR^{|I|\times |I|}$ is the submatrix obtained by selecting the rows and columns indexed by $I$ in $M$. For a symmetric matrix $M$, $\lambda_{\max}^+(M)$ is the largest positive eigenvalue and zero if they are all nonpositive. If $S$ is a set, $|S|$ denotes its cardinality.

\section{Gaussian Graphical Models with Latent Variables}

\label{sec:ggm}
We consider a multivariate Gaussian variable $(X_{O},X_{H})\in\RR^{p+h}$ where $O$ and $H$ are respectively the set of indices of observed variables, with $p=|O|$, and of latent variables, with $h=|H|$. We denote $\Sigma\in\RR^{(p+h)\times(p+h)}$ the complete covariance matrix and $K=\Sigma^{-1}$ the complete \emph{concentration matrix} or \emph{precision matrix}. Let $\hat{\Sigma}\in\RR^{(p+h)\times(p+h)}$ denote the empirical covariance matrix, based on a sample of size $n$. We only have access to the empirical marginal covariance matrix $\hat{\Sigma}_{OO}$. It is well known that the marginal concentration matrix on the observed variables can be computed from the full concentration matrix as
\begin{align}
\label{schur}
\Sigma_{OO}^{-1} = K_{OO}-K_{OH}K_{HH}^{-1}K_{HO}.
\end{align}
We assume that the original graphical model is sparse and that there is a small number of latent variables. 
This implies that  $K_{OO}$ is a sparse matrix and that $K_{OH}K_{HH}^{-1}K_{HO}$ is a low-rank matrix, of rank at most $h$. Note that $\Sigma_{OO}^{-1}$ is typically not be sparse due to the addition of the term $K_{OH}K_{HH}^{-1}K_{HO}$. \Cref{fig:graph} shows an example of an LVGGM structure where variables \{1,2,3\} are hidden variables and \Cref{fig:decomp}(a) shows the structure of its corresponding complete concentration matrix $K$. \Cref{fig:decomp}(b) shows an approximation of $\Sigma_{OO}^{-1}$ as ``sparse + low rank'' matrix.

\newcommand{\hnode}{-2}
\newcommand{\wlo}{1.5}

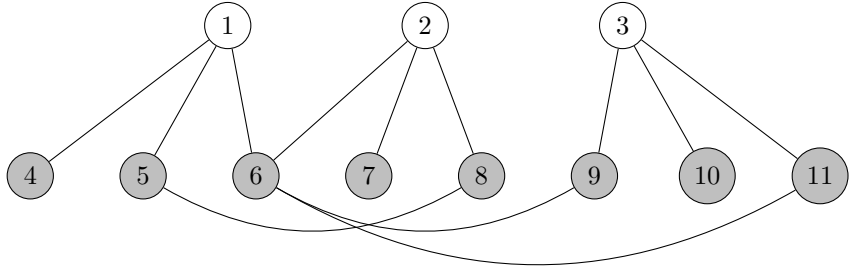
\begin{figure}[ht]
\center
\begin{tikzpicture}
\node[draw,circle,fill=white] (1)at(7/4*\wlo,0) {$1$};
\node[draw,circle,fill=white] (2)at(2*7/4*\wlo,0) {$2$};
\node[draw,circle,fill=white] (3)at(3*7/4*\wlo,0) {$3$};
\node[draw,circle,fill=gray!50] (4)at(0,\hnode) {$4$};
\node[draw,circle,fill=gray!50] (5)at(\wlo,\hnode) {$5$};
\node[draw,circle,fill=gray!50] (6)at(2*\wlo,\hnode) {$6$};
\node[draw,circle,fill=gray!50] (7)at(3*\wlo,\hnode) {$7$};
\node[draw,circle,fill=gray!50] (8)at(4*\wlo,\hnode) {$8$};
\node[draw,circle,fill=gray!50] (9)at(5*\wlo,\hnode) {$9$};
\node[draw,circle,fill=gray!50] (10)at(6*\wlo,\hnode) {$10$};
\node[draw,circle,fill=gray!50] (11)at(7*\wlo,\hnode) {$11$};
\draw (1) --(4);
\draw (1) --(5);
\draw (1) --(6);
\draw (2) --(6);
\draw (2) --(7);
\draw (2) --(8);
\draw (3) --(9);
\draw (3) --(10);
\draw (3) --(11);
\draw (5) to[bend right=30](8);
\draw (6) to[bend right=30](9);
\draw (6) to[bend right=30](11);
\end{tikzpicture}
\vspace{-2em}
    \caption{Example of an LVGGM structure where the variables $\{1,2,3\}$ are hidden variables}
    \label{fig:graph}
\end{figure}

\begin{figure}[ht]
\center
\begin{tabular}{cc}
  \includegraphics[width=.40\linewidth]{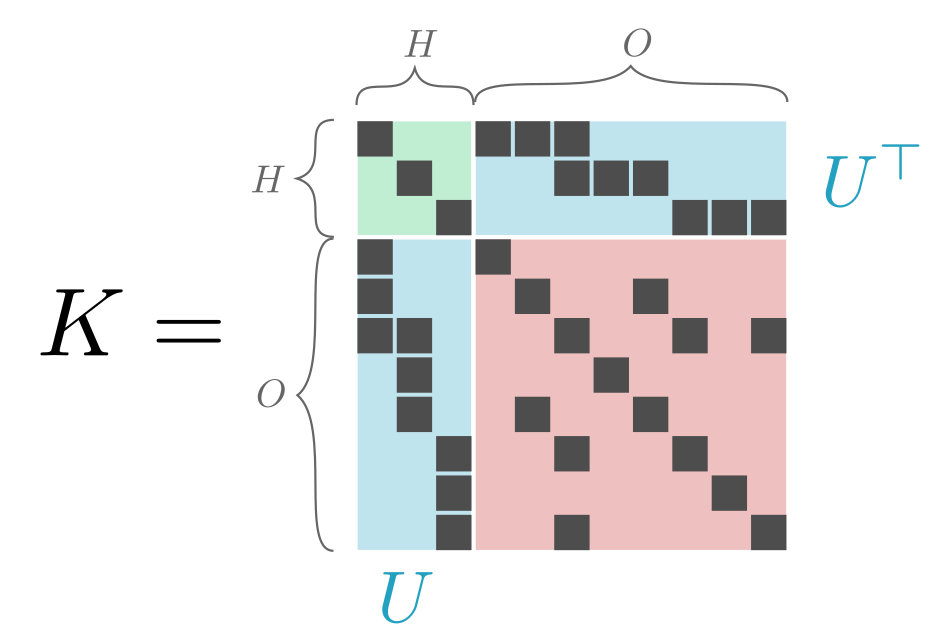}
&  \includegraphics[width=.59\linewidth]{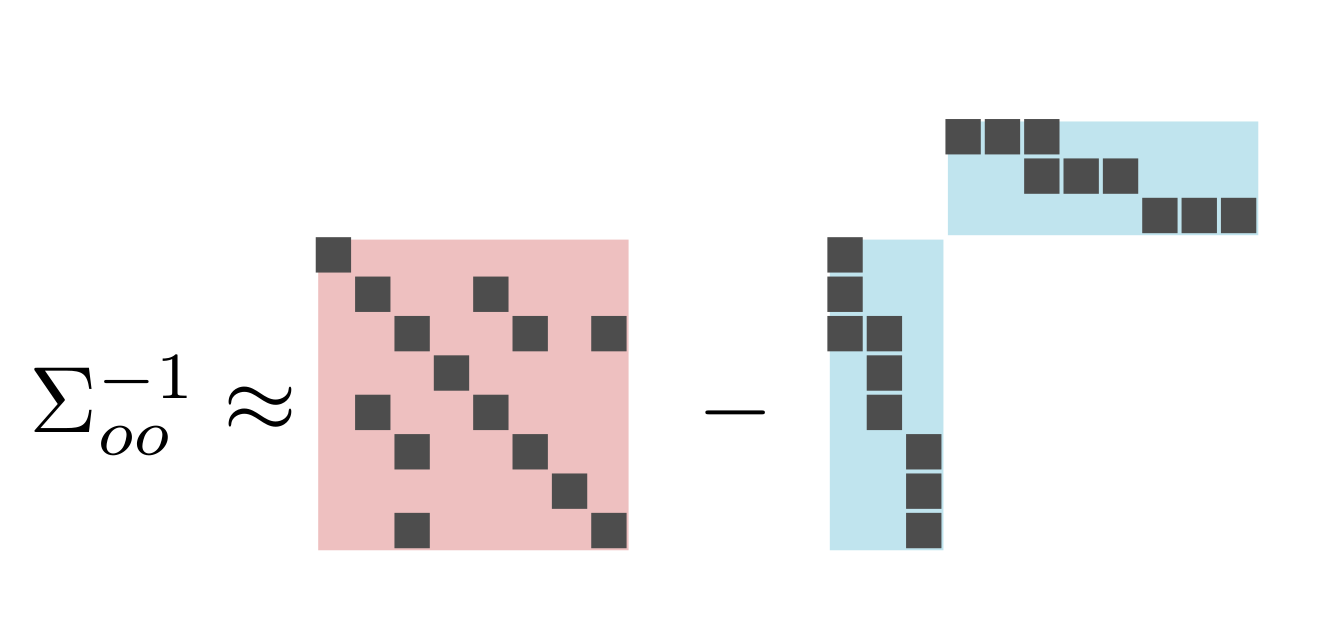}
\\    (a) & (b) \\[6pt]
\end{tabular}
  \vspace{-1em}
    \caption{(a) Structure of complete concentration matrix $K$ of graph in \Cref{fig:graph}. (b) approximation of $\Sigma_{OO}^{-1}$ as "sparse + low rank''}
    \label{fig:decomp}
\end{figure}

\citet{chandrasekaran2010} show that under appropriate conditions, namely if $K_{OO}$ is sufficiently sparse and $K_{OH}K_{HH}^{-1}K_{HO}$ is low rank and cannot be approximated by a sparse matrix, these two terms are identifiable and can be estimated, via an estimator of $\Sigma_{OO}^{-1}$ of the form $S-L,$ where $S$ is sparse, $L$ is low rank, and $S-L$, $S$ and $L$ are p.s.d. matrices in order to match the structure of \eqref{schur}, and guarantee that the estimate of the original matrix $K$ is p.s.d. Moreover the authors show that $S$ and $L$ can be estimated via the following convex optimization problem:

\begin{align}
\label{opt_tr}
&\min_{S,L} f(S-L)+\lambda\left(\gamma\|S\|_{1}+ \tr(L)\right) \\
&\quad \text{s.t.} \quad S-L \succeq 0, \quad L \succeq 0, \nonumber
\end{align}

where $f$ is a convex loss function,  and $\lambda,\gamma$ are regularization parameters. The positivity constraint on $S$ has been dropped since it is implied by $S-L \succeq 0$ and $L \succeq 0$. Typically, in GGM selection, $f$ is the negative log-likelihood.

\begin{align}
f_{ML}(M)&:=-\log\det(M) + \tr(M\hat{\Sigma}).
\end{align}

 
Two other natural losses, that have the advantage of being quadratic, are the second order Taylor expansion around the identity matrix of the log-likelihood $f_{T}$ and the score matching loss $f_{SM}$, introduced by \citet{hyvarinen2005estimation} and used for GGM estimation in \citet{lin2016estimation},
\begin{align}
f_{T}(M)&:=\frac{1}{2}\|\hat{\Sigma}^{1/2}M\hat{\Sigma}^{1/2}-I\|_2^2\\
f_{SM}(M)&:=\frac{1}{2}\tr(M^2 \hat{\Sigma})-\tr(M).
\end{align}

\citet{chandrasekaran2010} show that under appropriate technical conditions, the regularized maximum log-likelihood formulation (\ref{opt_tr}) provides estimates $(S_{n},L_{n})$ that have respectively the same sparsity pattern and rank as $K_{OO}$ and $K_{OH}K_{HH}^{-1}K_{HO}$. The obtained low rank component $L_{n}$ retrieves the latent variable subspace. 

Note first that, in general, $K_{HH}$ and  $K_{OH}$ are not identifiable and cannot be estimated from $L_{n}$. Therefore the connectivity between the latent variables and the connectivity between latent and observed variables cannot be recovered. However, under the assumption that the sources are conditionally independent given observed nodes, $K_{HH}$ is diagonal, and, when the groups of observed variables associated with each latent variables are moreover disjoint, the columns of $K_{OH}$ have disjoint support and are therefore orthogonal. This necessarily implies that they are proportional to the eigenvectors of $K_{OH}K_{HH}^{-1}K_{HO}$ as soon as the coefficients of the diagonal matrix $K_{HH}$ are all distinct, by uniqueness of the SVD. In that case, they are thus identifiable, and it makes sense to estimate the columns of $K_{OH}$ by the eigenvectors of the estimated $L$. 

However, if the columns of $K_{OH}$ are sparse, it would seem relevant to encode this in the model, as this is potentially a stronger prior than orthogonality. Moreover, it might be relevant to allow the groups of observed variables associated with each given latent variable to overlap.

In this work, assuming that the latent variables are independent,  we propose a formulation allowing to estimate the columns of $K_{HO}$ up to a constant, based on an assumption on its relative sparsity, that we encode as a prior using a matrix norm  introduced by \citet{richard2014tight}.

\section{Spsd-rank($k$) and a convex surrogate}
\label{sec:posrank}

\citet{richard2014tight} proposed matrix norms and gauges\footnote{We will use the word gauge in the paper to mean \emph{closed gauge}. We remind the reader that a closed gauge is simply a proper closed convex positively homogeneous function, and that a gauge $\gamma$ which is symmetric ($\gamma(x)=\gamma(-x)$), takes finite values, and such that $(\gamma(x)=0)\Rightarrow (x=0)$ is a norm. Gauges are thus natural generalizations of norms, that share many properties including the triangle inequality and the same Fenchel duality theory. We refer the reader to ~\citet{friedlander2014gauge} or ~\citet{rockafellar1970convex} for a more detailed presentation of gauges.} 
that yield estimates for low-rank matrices whose factors are sparse.  One variant, which is actually a gauge\footnote{See \citet{chandrasekaran2012convex} for a discussion.}, specifically suited  to the estimation of p.s.d.\ matrices, induces a decomposition into with sparse rank one p.s.d.\ factors. In this section, we introduce the $k$-spsd-rank of a p.s.d.\ matrix relate it to this gauge, which assumes that the sparsity of the factors is known and fixed. We then discuss a generalization for factors of different sparsity levels.\\

The following definition is a generalization of the rank for p.s.d.\ matrices,
\begin{definition}[$k$-spsd-rank] 
For a p.s.d.\ matrix $Z\in\RR^{p\times p}$ and for $k>1$ we define its $k$-spsd-rank as the optimal
value of the optimization problem:
\begin{align*}
&\min \|c\|_0 \\ 
&\text{s.t.} \enskip Z=\sum_{i} c_i u_i u_{i}^\top, \enskip c_i\in \RR^{+}, \enskip u_{i}\in\RR^p  :  \|u_{i}\|_0 \leq k, \|u_{i}\|_2 = 1.
\end{align*}
\end{definition}
Note that not all p.s.d.\ matrices admit such a decomposition, in which case the $k$-spsd-rank is by convention infinite. This is in particular the case for low-rank non sparse matrices like $11^{\top}$(see~\citet{richard2014tight} for a proof). A natural convex relaxation of the $k$-spsd-rank is based on the concept of \emph{atomic norm} proposed in \citet{chandrasekaran2012convex}. \emph{Atomic norms} are norms (or gauges) whose unit ball is the convex hull of a reduced set of elements of the ambient space $\mathcal{A}$ called \emph{atoms}. Here we consider the \emph{atomic gauge} associated with the set $\mathcal{A}=\{uu^\top \mid \|u\|_2 \leq 1, \: \|u\|_0 \leq k\}.$ In particular, it follows from basic results on \emph{atomic norms} that we can write this one as follows
\begin{definition}[$\Omega$, convex relaxation of $k$-spsd-rank] 
For $Z\in\RR^{p\times p},$ 
\begin{align*}
&\Omega(Z):=\min \|c\|_1 \\ 
&\text{s.t.} \enskip Z=\sum_{i} c_i u_i u_{i}^\top, \enskip 
c_i\in \RR^{+}, \enskip u_{i}\in\RR^p  :   \|u_{i}\|_0 \leq k, \|u_{i}\|_2 = 1.
\end{align*}
\end{definition}
Note that we can have $\Omega(Z)=+\infty$ even when $Z$ is p.s.d., if $Z$ cannot be decomposed in $k$-sparse,  rank-1 p.s.d.\ factors, as it is the case for $11^{\top}$. The polar gauge of $\Omega$ is characterized as follows:
\begin{lemma}
\label{lem:LMO}
Let $Y\in\RR^{p\times p}$ be a symmetric matrix. The polar gauge to $\Omega$ writes
\begin{align}
{\Omega^{\circ}}(Y)= \max_{I\in\mathcal{G}^p_k}\lambda^{+}_{max}(Y_{II}).
\end{align}
\end{lemma}

Unfortunately, the polar gauge $\Omega^{\circ}$ is a priori NP-hard to compute, since it is the largest sparse eigenvalue associated with a sparse eigenvector with $k$ non zero coefficients:

\begin{align*}
\min_u u^{\top}XX^{\top}u \quad \text{s.t.} \quad  \|u\|_0 \leq k,\quad \|u\|_2 = 1,
\end{align*}
which is known to be an NP-hard problem to solve \citep{moghaddam2008sparse}. 
However, a recent literature proposed quite a number of algorithms to solve sparse PCA approximately or heuristically, among others convex via relaxations~\citep{yuan2013truncated,d2008optimal,d2005direct}, which can be leveraged to approximately solve the corresponding problems.

\subsection{A variant for factors with different sparsity levels}
\label{sec:different_sparsity_levels}
$\Omega$ can be generalized to allow each rank one factor have a different sparsity level. 
A simple way to do this is to consider a gauge of the form
\begin{align*}
& \Omega_{w}(Z):=\inf \sum_{i}\sum_{k=1}^p w_{k}c_i^k  \\
& \text{s.t.} \enskip Z=\sum_{i}\sum_{k=1}^p c_i^k u_i^k u_{i}^{k\top},c_i^k\in \RR^{+}, u_{i}^k\in\RR^p  :   \|u_{i}^k\|_0 \leq k, \|u_{i}^k\|_2 = 1,
\end{align*}
where $k \mapsto w_{k}$ is an increasing function that penalizes each sparsity level $k$ by $w_{k}$. Via a simple change of variable, we can rewrite $\Omega_{w}$
\begin{align*}
& \Omega_{w}(Z):=\inf \sum_{i}\sum_{k=1}^p c_i^k  \\
& \text{s.t.} \enskip Z=\sum_{i}\sum_{k=1}^p c_i^k u_i^k u_{i}^{k\top},c_i^k\in \RR^{+}, u_{i}^k\in\RR^p  :   \|u_{i}^k\|_0 \leq k, \|u_{i}^k\|_2 = w_{k},
\end{align*}
which shows that it is a standard atomic gauge in which the rank one atoms with $k^2$ non-zero coefficients have weight $w_k$.
If we choose $w_k=1$ for all $k$, then it can be shown that only the non-sparse atoms will appear in the expansion and so $\Omega_{w}(Z)=\tr(Z)+\iota_{\{Z\succeq 0\}}.$ If $k\mapsto w_k$ accelerates quickly, the gauge will favor sparser factors, but since some p.s.d.\ matrices cannot be expressed as positive combinations of very sparse p.s.d.\ rank-one factors, the behavior of the gauge is not trivial for any weights of the form $w_k=k^m, \:m>0,$ even when $m$ is large. Although a detailed analysis of $\Omega_w$ is beyond the scope of this work, we illustrate this generalization in the experiments.


\section{Convex Formulation and Algorithm}
\label{subsec:alg}

We use $\Omega$ to impose structure on the low rank component and consider the following convex optimization problem,
\begin{align}
\label{opt}
\min_{S,L} f(S-L)+\lambda\big(\gamma\|S\|_{1}+\Omega(L)\big) \quad \text{s.t.} \quad S-L \succeq 0.
\end{align}
Note that the nonnegativity constraint on $L$ is no longer necessary since the gauge $\Omega$ only provides symmetric p.s.d. matrices, as a sum of p.s.d. rank-one matrices.\\

In order to rewrite our problem as a simple convex regularized by $\Omega$, we drop\footnote{It would be possible to still enforce $S-L \succeq 0$, with approach proposed in this paper using Lagrangian techniques with an increase of computational costs.} the nonegativity constraint on $S-L$ and consider the optimization problem
\begin{align}
\label{opt_nc}
\min_{S,L} f(S-L)+\lambda\big(\gamma\|S\|_{1}+\Omega(L)\big).
\end{align}

We propose the alternating optimization scheme presented in Algorithm \ref{alg:alt}. First, we update the sparse factor $S$ by optimizing problem (\ref{opt_nc}) with $L$ fixed, then we update $L$ by solving problem (\ref{opt_nc}) with $S$ fixed. 
\begin{itemize}
\item to update the sparse factor $S$ we apply a fixed number of soft-thresholding iterations, i.e several steps of iterative shrinkage-thresholding algorithm (ISTA). In the experiments we perform 10 soft-thresholding iterations when updating $S$
\item to update the low rank factor $L$ we apply an efficient algorithm for quadratic losses recently proposed by \citet{vinyes2017} called Fast Column Generation algorithm (FCG). This algorithm is well adapted to the quadratic losses $f_{T}$ and $f_{SM}$ introduced in Section \ref{sec:ggm}
\end{itemize}

FCG  consists in applying a Fully Corrective Frank Wolfe \citep{LacosteFCFW} to a regularized optimization problem. Frank Wolfe (FW) algorithm \citep{frank1956algorithm}, also known as conditional gradient, is particularly well suited for solving quadratic programming problems with linear constraints. They apply in the context where we can easily solve the Linear Minimization
Oracle (LMO), a linear problem on a convex set of constraints $\mathcal{C}$ defined as
\begin{align}
{{\rm LMO}}_{\mathcal{C}}(y) := \arg\min_{z \in \mathcal{C}} \left\langle y,z \right\rangle.
\end{align}
In particular $\mathcal{C}$ can be the convex hull of a set of atoms $\A$. At each iteration FW selects a new atom $a^t$ from $\mathcal{C}$ querying the LMO and computes the new iterate as a convex combination of $a^t$ and the old iterate $x^t$. The convex update can be done by line search. FCFW, discussed in \citet{LacosteFCFW}, is a variant of FW  that consists in finding the convex combination of all previously selected atoms $(a^i)_{i<t}$. When using the algorithm proposed in \citet{vinyes2017} we need to compute the following LMO 
\begin{align}
{{\rm LMO}}_{\Omega}(M) :=\arg\max_{u}\: u^{\top}Mu \quad \st \quad \|u\|_0=k, \|u\|_ 2=1.
\end{align}
at each iteration, and subsequently use a working set algorithm to solve the fully corrective step.\\

We propose to use the Truncated Power Iteration (TPI) heuristic introduced by \citet{yuan2013truncated} to obtain an approximation to the oracle ${{\rm LMO}}_{\Omega}(M)$.

\begin{algorithm*}
\caption{Alternate minimization}
\label{alg:alt}
\begin{algorithmic}[1]
\State\textbf{Require: } $f$ quadratic, maximum iterations $T$ 
\State\textbf{Initialization: } $S^{0}=0$, $L^{0}=0$, $t=0$
\For{$t=1..T$}
\State Compute $S^{t}$ applying a fixed number of \texttt{ISTA} iterations on problem (\ref{opt_nc}) with $L^{t-1}$ fixed
\State Compute $L^{t}$ applying \texttt{FCG} on problem (\ref{opt_nc}) with $S^{t}$ fixed
\EndFor
\State return $S^{t}, L^{t}$
\end{algorithmic}
\end{algorithm*}

\section{Identifiability of $S^*$ and of the sparse factors of $L^*$}
\label{sec:id}

For formulation \eqref{opt_nc} to yield good estimators, a necessary condition is that, if $M$ is a marginal precision matrix with decomposition $M=S^*+L^*$ with $L^*=\sum_{i} s_i u^i {u^i}^{\top}$, $\supp(u^i)\subset I_i$ and $|I_i|=k$, this decomposition can be recovered from perfect knowledge of $M$ (which corresponds to the case where we have an infinite amount of data with no noise). We therefore consider in this section the decomposition problem of a known precision matrix $M$. For the estimator obtained from \eqref{opt_nc} to provide reasonable estimates, a necessary condition is that it returns correct estimates in the limit of an infinite amount of data. 

We will provide sufficient conditions on $S^*$ and $L^*$ so that if $M=S^*+L^*$ and $(\hat{S}, \hat{L})$ is an optimum of the problem
\begin{align}
\label{pb:main}
\min \gamma \|S\|_1+\Omega(L) \quad \text{\st} \quad M=S+L,
\end{align}
then $\hat{S}=S^*$, $\hat{L}=L^*$ and the decompositions of $\hat{L}$ and $L^*$ are the same. Our approach is based on the work of~\citet{chandrasekaran2011rank} but several of  our results and proofs are tighter than the original analysis. \\

We will make the simplifying assumption that the sets $I_i$ are disjoint, so that part of the analysis decomposes on each of the blocks $I_i \times I_i$ and on the complement of $\bigcup_i I_i \times I_i$.
\begin{assumption}
\label{as:disjoint}
Let $L^*=\sum_{i} s_i u^i {u^i}^\top$, with $\supp(u^i)=I_i$. We assume that the sets $I_i$ are all disjoint and that $|I_i|=k$.
\end{assumption}
In particular, this assumption entails implicitly that if $L^*=\sum_{i} L_i^*$ with $L_i$ the component supported on block $I_i\times I_i$, then $L_i^*$ is of rank one.\\

In order to be able to decompose $M$ as $M=S^*+L^*$, we need to make assumptions on $S^*$ and $L^*$.  Indeed, there are a number of scenarios in which the possible decompositions of $M$ into \emph{psd rank-one} matrices  and \emph{sparse} parts may not be uniquely defined. For instance if the low-rank matrix is itself sparse, or the sparse part not sufficiently sparse, the decomposition might not be identifiable.\\
Two quantities are key: let $\tauu$ be an upper bound such that
$$\tauu \geq k \max_{i \in \itgset{r}} \|u^i\|_{\infty}^2 \quad \text{and} \quad  k_0:=\max_{i}\|S^*_{i\cdot}\|_0, \: \text{where} \:\|S^*_{i\cdot}\|_0:=\big |\{j \mid S^*_{ij} \neq 0\} \big|.$$
On one side, $k_0$ measures the sparsity of $S^*$, it is the maximal degree of the graph on the observed variables. $S^*$ will be sufficiently sparse if $k_0\ll k.$ On the other, $\tauu\geq 1$ measures the flatness (vs spikiness) of $L^*$: again $L^*$ be sufficiently flat if $\tauu \ll k.$

The interpretation behind an assumption of the form $k_0\ll k$  is that, in the precision matrix of the joint distribution over observed and latent variables, all the neighbors of a latent node $i$ form a clique, and in this clique,
each node has $k$ neighbors. If $k_0\ll k$, then the connections explained by this clique cannot be attributed to individual connections between observed nodes, and can only be attributed to the presence of a latent variable. 

Second, the interaction strength of each hidden node $i$ with its observed neighbors in the graph should be of a similar order of magnitude. 
Symmetrically, an assumption of the form $\tauu \ll k$ just imposes an upper bound on  the interaction strength between a hidden node and its observed neighbors. Indeed, if latent node $i$ had very strong interactions with $j$ and $j'$, in the marginalized graph the interaction between $j$ and $j'$ induced by $i$ might be difficult to tell appart from a direct interaction between $j$ and $j'$.\\

In the next theorems, we will either assume that $\alpha:=k_0 \sqrt{\frac{2\tauu}{k}}$, which combines both quantities, is small, or, that $k_0\leq \frac{1}{7} \sqrt{k}$ and $\tauu\leq 2.$

To be able to position our general result w.r.t. to the literature, we first state a counterpart for the decomposition into a sparse and a (non necessarily) sparse rank-one p.s.d.\ matrix, which is very close but improves Corollary 3 of \citet{chandrasekaran2011rank}.

\begin{theorem}[sparse + one rank-one block]
\label{theo:chand}
Let $M=S^*+L^*.$

Consider the optimization problem 
\begin{equation}
\label{eq:l1_plus_tr}
\min \gamma \|S\|_1+{\tr}(L) \quad \text{\st} \quad M=S+L,\quad L \succeq 0.
\end{equation}

Under the assumption that $L^*$ is p.s.d., rank one and symmetric,
if, for the pair $(S^*,L^*)$ the quantities $k_0,p$ and $\tauu$ are such that $\alpha:=k_0\sqrt{\frac{2\tauu}{p}}$ satisfies $\alpha+\frac{\alpha^2}{2k_0}<\frac{1}{3},$ where $p$ is the ambient dimension,
there exist values of $\gamma$, such that 
\begin{equation}
\label{eq:gamma_interval}
\frac{\bar{\tau}}{p}\frac{ 1}{1-3\alpha} \leq \gamma < \frac{1}{k_0} \frac{1-k_0 \tauu/p}{1 + \alpha},
\end{equation}
(i.e.\ the interval is non empty), and, for any such value of $\gamma$, the pair $(S^*,L^*)$ is the unique optimum of problem \eqref{pb:main}.\end{theorem}

The result we obtained here provides an improvement over the main result in~\citet{chandrasekaran2011rank} as stated in Corollary 3. Indeed, in our setting (a single rank one component), the quantities appearing in that result can be computed: $\text{deg}_{\max}(S^*)=k_0$ and $\text{inc}(L^*)=\sqrt{\frac{\tauu}{k}}.$ Thus Corollary 3 of~\citet{chandrasekaran2011rank} requires  $\alpha< \frac{\sqrt{2}}{12}$ when $\alpha<\frac{2}{7}$ is sufficient in our case, and even smaller values of $\alpha$ are allowed for sufficiently large $k_0$; also, the interval allowed for $\gamma$ in~\citet{chandrasekaran2011rank} is, with our notations, $\Big ( 2 \sqrt{\frac{\tauu}{k}} (1-8 \alpha/\sqrt{2})^{-1},  \frac{1}{k_0}(1-6 k_0\sqrt{\frac{\tauu}{k}}) \Big),$ where both the upper bound and the lower bound have a dependence in $\sqrt{\frac{\tauu}{k}}$, while we obtain a dependance in $\frac{\tauu}{k}.$
Given that~\citet{chandrasekaran2011rank} show that there always exist a value of $\gamma$ that is valid under the assumption that $\alpha< \frac{\sqrt{2}}{12}$, this improvement might seem minor, but since $\gamma$ depends on quantities that are not known in practice and need to found by trial and error, knowing that a larger interval is allowed might help finding a correct value of $\gamma$ in practice.
Note that this improvement is not due to the fact that we restricted ourselves to the rank one case, but to the use of sharper \emph{incoherence measures} (see Definition~\ref{def:xxi}) and improvements in the bounding scheme for the subgradients.

In fact, the possibility of choosing a value of $\gamma$ which is an order of magnitude smaller is crucial for the theorem that we present next, and which extends this type of result to the recovery of several sparse p.s.d.\ rank one terms, using the gauge $\Omega.$

\begin{theorem}[sparse + multiple \emph{sparse} rank-one blocks]\vspace{1mm}
\label{theo:two}
Let $\alpha := k_0\sqrt{2\tauu/k}$ and let $\mu:=(1-3\alpha)^{-1}$. Under Assumption \ref{as:disjoint}, if $k_0\leq \frac{1}{7}\sqrt{k},$ and if there exists $\kappa>16 \mu$ and $\taul,\tauu>0$ such that $\taul+\tauu=2$, with 
\begin{equation}
\label{eq:tau_constraints}
\kappa \tauu^2 \frac{k_0}{k} < \taul \leq 1 \quad \text{and} \quad \forall j\in I_i,  \quad \frac{\underline{\tau}}{k}\leq (u_j^i)^2 \leq \frac{\bar{\tau}}{k},
\end{equation}
then there exists a constant $C>0$ such that if $k>Ck_0$, the pair $(S^*,L^*)$ is the unique optimum of problem \eqref{pb:main} for a regularization parameter $\gamma := \mu\frac{\tauu}{k}$.
\end{theorem}

Note that $\tauu$ is essentially the same upper bound as before, except that it is now tied with a lower bound $\taul$; these constrained are however relaxed when $C$ is sufficiently large, and $\taul$ can then be chosen sufficiently small to allow for all lower bounds to hold.

\subsection{An informal motivation for the tangent space based analysis}
As first discussed in \citet{chandrasekaran2011rank} and later in \citet{negahban2012unified}, specific subspaces play a natural role in the analysis of this type of decomposition problem. 

Consider first a simple sparse $\scriptstyle +$ low-rank decomposition of a matrix $M=S^*+L^*$.
If the decomposition is unique, then by definition there is no perturbation $(\Delta S,\Delta L)$ so that (a) $S^*+\Delta S$ has the same sparsity pattern as $S^*$, (b) $L^*+\Delta L$ is of rank $r$, and (c) $M=S^*+\Delta S+L^*+\Delta L$. Note that we then have $\Delta S+\Delta L=0$. We continue this discussion informally to provide intuition. 
A particular case occurs is if this equality holds for an infinitesimal pair $(\Delta S,\Delta L)$, in which case $\Delta S$ and $\Delta L$ must each belong respectively to a certain tangent set: indeed, since $L^*+\Delta L$ belongs to the manifold of matrices of rank $k$, then in the limit of small $\Delta L$, it belongs to the tangent space to the manifold of rank $k$ matrices at $L^*$, a space which we will denote $\T_r(L^*)$; 
for $S^*$ the assumption that $S^*$ has $s$ non zero coefficients is equivalently reformulated as the constraint that $S$ belong the union of all the subspaces spanned by $s$ elements of the canonical basis, which is a union of manifolds. 
In particular, if $S^*$ has exactly $s$ non zero coefficients, this fixes the support, which has to contain the support of $\Delta S$. 
Since $S^*$ is in a manifold which is simply a linear subspace, then $\Delta S$ must belong to that subspace as well, which we can denote $\T_s(S^*)$ and call the tangent space for $S^*$. 
To exclude the existence of non trivial pairs $(\Delta S,\Delta L)$ such that $\Delta S+\Delta L=0$, it seems relevant to impose that $\T_s(S^*)\cap \T_r(L^*)=\{0\}$, i.e.\ the subspaces are in~\emph{direct sum}. If this equality holds,~\citet{chandrasekaran2011rank} say that the subspaces are \emph{transverse}. 

The previous discussion is non-rigorous because we reasoned informally about infinitesimal $(\Delta S,\Delta L)$. What \citet{chandrasekaran2011rank}
have shown is that if we solve $\min_{(S,L)} \|S\|_1+\|L\|_{
\rm tr} \hspace{1.5mm} \st \hspace{1.5mm} M=S+L$, then, for a solution $(\hat{S},\hat{L})$, the first order optimality conditions of this optimization problem naturally decompose onto $\T_s(\hat{S})$, $\T_r(\hat{L})$ and their orthogonal complements. This type of decomposition of optimality condition on a tangent space and its complement motivated the introduction the term \emph{decomposable norm} in \citet{negahban2012unified}.

In our case, $L$ is not simply low rank, it is a sum of p.s.d.\ matrices $L_i$ of rank $r_i$ each with support in $I_i \times I_i$. We will therefore have to consider the tangent subspaces to the manifolds associated with each $L_i$. 

\subsection{Definition of tangent spaces and associated projections}

For a symmetric sparse matrix $S$, let $\T_s(S)$ be the tangent space at $S$ with respect to the set of symmetric sparse matrices:
$$\T_s(S)=\{M \in \RR^{p \times p}\mid  \: M=M^\top, \: \supp(M)\subset \supp(S)\}.$$
Next, let $\T_I(u)$ be the tangent space at $uu^{\top}$ to the  manifold of rank one matrices, restricted to the space of matrices with support in $I\times I$. 
If we first define $\bar{\T}_I$, the subspace of matrices with support included in $I\times I$ with  $$\bar{\T}_I:=\{M \in \RR^{p \times p}\mid  M=M^\top, \: \supp(M)\subset I \times I\},$$
then, as in~\citet{chandrasekaran2011rank},  we can express concisely $\T_I(u)$ as 
$$\T_I(u):=\{M \in \bar{\T}_I\mid \: M=uv^{\top}+vu^{\top},\: v\in \RR^{p} \}.$$
Let $\T^c_s(A)$ denote the orthogonal complement of  $\T_s(A)$ in $\RR^{p \times p}$ and $\T^c_I(u)$ denote the orthogonal complement\footnote{Note in particular that it is not the orthogonal complement in the entire space.} of $\T_I(u)$ in $\bar{\T}_I$. 

The projections on the defined subspaces are respectively $\mathcal{P}_{\T_s(A)}(M)=M_{\supp(A)}$ and $\mathcal{P}_{\T_I(u)}(M)=\mathcal{P}_{u}(M_{II})$ with $$\mathcal{P}_u(M):=M-(I-uu^\top)M(I-uu^\top).$$

In order to simplify notations we introduce 
$$\T_0:=\T_s(S^*), \quad \T_i :=\T_{I_i}(u^i),\quad \bar{\T}_i :=\bar{\T}_{I_i}, \quad \bar{\T}_{00}:=\T_0 \cap \text{span}\big((\bar{\T}_i)_{i \in \itgset{r}} \big )^\bot.$$

\subsection{First order optimality conditions}

Since~\eqref{pb:main} is a convex optimization problem, its minima are characterized by first order subgradient conditions. The pair $(S^*,L^*)$ with $L=\sum_i {s_i u^i {u^i}^\top}$ is an optimum of \eqref{pb:main} if and only if an only if there exists a dual $Q$ satisfying first order optimality conditions

\begin{align*}
Q\in \gamma \partial \|.\|_1(S^*) \quad \text{and} \quad Q\in \partial \Omega(L^*).
\end{align*}

With the introduced tangent spaces, we state the following proposition that provides sufficient conditions for the existence of a unique optimum of \eqref{pb:main}.

\begin{proposition}
\label{lem:uniqueopt}
The pair $(S^*,L^*)$  is the unique optimum of \eqref{pb:main} if
\begin{enumerate}[label=\textbf{(T)}, leftmargin=4\parindent]
\item \label{cond:transversality} $\forall i \in \itgset{r},\quad \T_0\cap \T_i = \{0\}$, 
\end{enumerate}
and there exists a dual $Q\in\RR^{p\times p}$ such that:
\begin{enumerate}[label=\textbf{(S.\arabic*)}, leftmargin=4\parindent]
\item \label{cond:l1_a} $\mathcal{P}_{\T_0}(Q)=\gamma \, \sign(S^*)$
\item \label{cond:l1_b} $\|\mathcal{P}_{\T^c_0}(Q)\|_{\infty}<\gamma$
\end{enumerate}
\begin{enumerate}[label=\textbf{(L.\arabic*)}, leftmargin=4\parindent]
\item \label{cond:om_1} $\forall i \in \itgset{r},\quad\mathcal{P}_{\T_i}(Q)=u^i{u^i}^\top$
\item \label{cond:om_2} $\forall i \in \itgset{r},\quad \lambda_{\max}^+ \big (\P_{\T_i^c}(Q) \big )<1$
\item \label{cond:om_3} $\forall J \in \mathcal{G}^p_k \backslash \{I_1,\ldots, I_r\}, \quad \lambda_{\max}^+ (Q_{\!J\!J})< 1$
\end{enumerate}
\end{proposition}

Note that the optimality condition decompose on the subspaces of matrices with support in the sets $I_i \times I_i$ and in the remaining set of indices, the complement of $\bigcup_i I_i \times I_i$. Indeed, we can write  $Q=\sum_{i=1}^n Q_{I_iI_i}+Q_{0,0}$ where $Q_{0,0}$ is the matrix whose non-zero coefficients are the coefficients of $Q$ that are not indexed by any pair in $\bigcup_{i=1}^r I_i \times I_i$. If $Q \in \text{span}(\T_0,\ldots,\T_r),$ then, we necessarily have $Q_{I_iI_i} \in \text{span}(\T_0,\T_i)$ and if 
$\T_0\cap \T_i=\{0\}$ then $Q_{I_iI_i}$ admits a unique decomposition $Q_{I_iI_i}=Q_i+Q_{i,0}$ with $Q_i \in \T_i$ and $Q_{i,0} \in \T_0 \cap \bar{\T}_i.$
\subsection{Transversality and incoherence conditions}

Since we consider a convex formulation, transversality is not sufficient: we need more than an assumption that $\T_0 \cap \T_i=\{0\}$ for all $i$. In fact, it will be necessary to assume that $\T_0$ and $\T_i$ are not too far from being orthogonal subspaces, a property which is usually called \emph{incoherence} \citep{tropp2004just,candes2009exact,chandrasekaran2011rank}. And furthermore, it will be necessary that elements of one subspace do not have a too large norm for the norm associated w.r.t. to another subspace.

\begin{definition}[Incoherence measures]
\label{def:xxi}
For $i$ in $\itgset{r}$, let 
\begin{eqnarray*}
\xxi_{i\rightarrow 0}&=&\max \{ {\|M\|_\infty} \mid M \in \T_i, \: \|M\|_{\op}\leq 1\},\\
\xxi_{0\rightarrow i}&=&\max \{ {\|Z\|_\op} \mid Z \in \T_0, \: \|Z\|_{\infty}\leq 1\},\\
\xxi_{i\rightarrow 0}'&=&\max \{ {\|\mathcal{P}_{\T_0}(M)\|_\infty} \mid M \in \T_i, \: \|M\|_{\op}\leq 1\},\\
\xxi_{0\rightarrow i}'&=&\max \{ {\|\mathcal{P}_{\T_i}(Z)\|_\op} \mid Z \in \T_0, \: \|Z\|_{\infty}\leq 1\}.\\
\end{eqnarray*}
\end{definition}

Note that by definition $\xxi'_{i\rightarrow 0} \leq \xxi_{i\rightarrow 0}$ and 
$\xxi'_{0\rightarrow i} \leq 2 \xxi_{0\rightarrow i}$. For this reason \citet{chandrasekaran2011rank} only introduced quantities of the type $\xxi_{i\rightarrow j}$. However, given that they involve the projection of one subspace on another, the quantities $\xxi'_{i\rightarrow j}$ are the ones that really capture that the subspaces are incoherent, whereas $\xxi_{i\rightarrow j}$ is an measure of incoherence between a subspace and a norm. The quantity $\xxi'_{i\rightarrow j}$ can be much smaller than 
$\xxi_{i\rightarrow j}$, so the distinction is useful.

\begin{lemma}[Bounds on $\xxi$]
\label{lem:bounds_xxi}
$$
\xxi_{i\rightarrow 0}' \leq \xxi_{i\rightarrow 0} \leq  \sqrt{\frac{2\bar{\tau}}{k}}, 
\quad \xxi'_{0\rightarrow i}\leq 2 k_0 \sqrt{\frac{k_0 \bar{\tau}}{k}}  
\quad \text{and} \quad\xxi'_{0\rightarrow i}\leq \xxi_{0\rightarrow i} \leq k_0.
$$
\end{lemma}

We then have

\begin{lemma}[Transversality]
\label{lem:transversality}
Let $\alpha:=k_0\sqrt{\frac{2\tauu}{k}}$. If $\alpha<1,$ then, for all $i \in\itgset{r}$, $\T_0\cap \T_i = \{0\}$.
\end{lemma}

\section{Proofs of main theorems}
We will first prove  Theorem \ref{theo:two} and then use some of the intermediate results to prove the restricted case of Theorem \ref{theo:chand}. For proofs of the different lemmas and propositions we refer the reader to the supplementary material.\\

\subsection{Proof of Theorem~\ref{theo:two}}

Notice that the assumptions that $k_0< \frac{1}{6}\sqrt{k}$ and that $\tauu\leq 2$ together imply that we have $\alpha < 1/3$. In order to prove this theorem we aim to construct a dual $Q\in\text{span}\{\T_0,\T_1,...,\T_r\}$ satisfying \ref{cond:l1_a}, \ref{cond:l1_b}, \ref{cond:om_1}, \ref{cond:om_2} and \ref{cond:om_3} of Proposition~\ref{lem:uniqueopt}.
We can write any matrix  $Q\in\text{span}(\T_0,\T_1,...,\T_r)$ as $Q=\sum_{i=1}^n Q_{I_iI_i}+Q_{0,0}$ where $Q_{0,0}$ is the matrix whose non-zero coefficients are the coefficients of $Q$ that are not indexed by any pair in $\cup_{i=1}^r I_i \times I_i$. But by Lemma \ref{lem:transversality}, $\forall i \in \itgset{r},\; \T_0\cap \T_i = \{0\}$, which entails that $Q_{I_iI_i}$ admits a unique decomposition $Q_{I_iI_i}=Q_{i}+Q_{i,0}$ with $Q_i \in \T_i$ and $Q_{i,0} \in \T_0.$ Finally, given the difference of supports, $Q_{0,0}$ is clearly orthogonal to $\text{span}\{\T_1,...,\T_r\}$ which entails that  $Q_{0,0} \in \T_0$. As a consequence, if we define $Q_0:=Q_{0,0}+\sum_{i=1}^r Q_{i,0}$, then $Q=\sum_{i=0}^r Q_i$ provides the unique decomposition of $Q$ such that $Q_i \in \T_i$ for all $i$. 

In the next part of this proof, we consider a number of projectors and other linear transformations operating on the $Q_i$s. Since some of these calculations are naturally written in matrix form, it is most natural to view the $Q_i$s as vectors. For the sake of clarity, we therefore switch notations and write $q_i$ for a vectorization of $Q_i$, and $q$ for a vectorization of $Q$. We slightly abuse notation and still say that $q_i$ belongs to $\T_i$, identify it with the corresponding matrix, etc. We also write $P_{\T_i}$ the matrix of the projector $\mathcal{P}_{\T_i}$ in the same basis as the one in which $q_i$ is written.

With this change of notation, $q$ is uniquely decomposed onto $\T_0\oplus\T_1\oplus...\oplus\T_r$ and we can write 
\begin{align}
\label{eq:q_decomp}
q = \sum_{i=0}^r ( q_i^* + \varepsilon_i),
\end{align}
where $q_0^* = \gamma \sign(S^*)$, $q_i^* = u^i{u^i}^{\top}$ for $i\in\itgset{r}$ and 
$\varepsilon_i \in \T_i$ for $i\in\{0,1,\ldots,r\}$.
Conditions \ref{cond:l1_a} and \ref{cond:om_1} are satisfied if and only if $P_{\T_i} q=q_i^*$ for all $0\leq i \leq r$, which is true if and only if $(\varepsilon_i)_{1 \leq i \leq r}$ solves the following system of equations:
\begin{align*}
\left\{
    \begin{array}{ll}
        \varepsilon_0  + \sum_{i=1}^r P_{\T_0}q_i^* + P_{\T_0}\varepsilon_i =  0,
        \vspace{1em}\\
        P_{\T_i}q_0^* + P_{\T_i}\varepsilon_0  + \varepsilon_i = 0, \quad \forall i\in\itgset{r}.
    \end{array}
\right.
\end{align*}

Denote $\varepsilon_{0,i}:=\mathcal{P}_{\bar{\T}_{I_i}}(\varepsilon_0)$ the projection of $\varepsilon_{0}$ on the set of matrices with support in $I_i\times I_i$. Note that we always have $P_{\T_i}\varepsilon_0=P_{\T_i}\varepsilon_{0,i}$, because $\T_i$ is a subspace of $\bar{\T}_{I_i}$. Finally, note that we have $\varepsilon_0=\sum_{i=1}^r \varepsilon_{0,i}$ because, by projecting the first equation above onto the subspace $\bar{\T}_{00}$ of matrices with zero entries on $\bigcup_{i=1}^r I_i \times I_i.$, we get $\P_{\bar{\T}_{00}}\varepsilon_0=0$.

Since the sets $I_i$ are disjoint, by projecting on the each of the spaces of matrices with support in $I_i\times I_i$ the previous system of equations, we get the equivalent set of systems:
\begin{align}
\label{eq:sys_eqs}
\forall i \in \itgset{r}, \quad 
\begin{bmatrix}
I & P_{\T_0}\\
P_{\T_i} & I\\
\end{bmatrix} 
\begin{bmatrix}
\varepsilon_{0,i} \\
\varepsilon_i \\
\end{bmatrix} 
&= \begin{bmatrix}
\eta_0\\
\eta_i\\
\end{bmatrix}
\quad \text{where} \quad
\begin{bmatrix}
\eta_0\\
\eta_i\\
\end{bmatrix}
= \begin{bmatrix}
-P_{\T_0}q_i^*\\
-P_{\T_i}q_0^*\\
\end{bmatrix},
\end{align}

The following lemma provides conditions for the invertibility of \eqref{eq:sys_eqs} and the form of the inverse matrix.
\begin{lemma} 
\label{lem:inverse}
Let $A:=\begin{bmatrix}
I & P_{\T_0}\\
P_{\T_i} & I\\
\end{bmatrix}$.\\ Then, with Definition~\ref{def:xxi}, if $\xxi_{0\rightarrow i}\xxi_{i\rightarrow 0} \leq \alpha< 1,$ $A$ is invertible and its inverse is 
$$
A^{-1}=
\begin{bmatrix}
I & -P_{\T_0}\\
-P_{\T_i} & I\\
\end{bmatrix}
\begin{bmatrix}
(I-P_{\T_0}P_{\T_i})^{-1} & 0\\
0 & (I-P_{\T_i}P_{\T_0})^{-1}\\
\end{bmatrix}.
$$
Moreover, 
$\qquad \displaystyle
\begin{cases}\forall v \in \T_i, \quad &\|(I-P_{\T_i}P_{\T_0})^{-1} v\|_{\op} \leq {\textstyle \frac{1}{1-\alpha}} \|v\|_{\op},\\
\forall v \in \T_0, \quad &\|(I-P_{\T_0}P_{\T_i})^{-1} v\|_{\infty} \leq {\textstyle \frac{1}{1-\alpha}} \|v\|_{\infty}.
\end{cases}$

\end{lemma}
But if we let $\alpha:=k_0\sqrt{\frac{2\tauu}{k}},$ then by Lemma~\ref{lem:bounds_xxi}, we have $1- \xxi_{0\rightarrow i}\xxi_{i\rightarrow 0} \geq 1 - \alpha$ and the assumption that $k_0< \frac{1}{6}\sqrt{k}$ entails that $\alpha< \frac{1}{3}<1$, so, by the previous lemma, each of the systems in \eqref{eq:sys_eqs} has a unique solution, and the obtained $(\varepsilon_i)_{i \in \itgset{r}}$ together with $\varepsilon_0=\sum_{i=1}^r \varepsilon_{0,i}$ thus yield in \eqref{eq:q_decomp} a value of $q$ that satisfies conditions \ref{cond:l1_a} and \ref{cond:om_1}.\\

We now prove that this value of $q$ satisfies \ref{cond:l1_b} and \ref{cond:om_2}, which requires to bound $\|P_{\T_0^c}q\|_{\infty}$ and $\Omega^{\circ}(P_{\T_i^c}q)$. Since $\Omega^{\circ}(P_{\T_i^c}q) \leq \|P_{\T_i^c}q\|_{\op}$, we bound this latter quantity.

\begin{lemma}[Bounds on $\|P_{\T_0^c}q\|_{\infty}$ and $\|P_{\T_i^c}q\|_{\op}$]
\label{lem:bounds_projq_help}
 Assume $\xxi_{0\rightarrow i}\,\xxi_{i\rightarrow 0} \leq \alpha<1,$ and let $q$ be defined by~\eqref{eq:q_decomp}, with $\varepsilon_0=\sum_{i \in\itgset{r}} \varepsilon_{0,i}$ and the pairs $(\varepsilon_{0,i},\varepsilon_{i})$ the unique solution of~\eqref{eq:sys_eqs}. Then
\begin{align*}
\|P_{\T_0^c}q\|_{\infty} 
\leq \max_{i\in\itgset{r}}\|q_i^*\|_\infty + \xxi_{i\rightarrow 0}\|\varepsilon_i \|_{\op} \quad \text{and} \quad
\|P_{\T_i^c}q\|_{\op}
\leq  \|q_0^*\|_{\op} + \xxi_{0\rightarrow i}\|\varepsilon_0 \|_{\infty}.
\end{align*}
\end{lemma}
The following lemma provides upper bounds for the quantities $\|\varepsilon_0\|_{\infty}$ and $\|\varepsilon_i\|_{\op}$.
\begin{lemma}[Bounds on $\varepsilon_i$]
\label{lem:bounds_eps}
 If $\xxi_{0\rightarrow i}\,\xxi_{i\rightarrow 0} \leq \alpha<1,$ and $(\varepsilon_i)_{i\in\itgset{r}}$ be defined as in the previous lemma, then
\begin{align*}
&\|\varepsilon_0\|_{\infty} \leq {\textstyle \frac{1}{1-\alpha}}\big(\frac{\bar{\tau}}{k} + \xxi_{i\rightarrow 0}'2\gamma k_0\big) \qquad \text{and}\qquad 
\|\varepsilon_i\|_{\op} \leq  {\textstyle \frac{1}{1-\alpha}}\big(2\gamma k_0 + \xxi_{0\rightarrow i}'\frac{\bar{\tau}}{k}\big).
\end{align*}
\end{lemma}
Finally we obtain simplified bounds on $\|P_{\T_0^c}q\|_{\infty}$ and $\|P_{\T_i^c}q\|_{\op}$.
\begin{lemma}[Simplified bounds on $\|P_{\T_0^c}q\|_{\infty}$ and $\|P_{\T_i^c}q\|_{\op}$]
Let $\alpha:=k_0\sqrt{\frac{2\tauu}{k}}$. If $\alpha<1$, for $q$ as in Lemma~\ref{lem:bounds_projq_help}, we have
\label{lem:bounds_projq}
\begin{align*}
\|P_{\T_0^c}q\|_{\infty} &\leq 
 \frac{\bar{\tau}}{k}\frac{ 1- \alpha + \alpha^2\sqrt{2/k_0}}{1-\alpha}   + \gamma \frac{2 \alpha}{1-\alpha}, \quad
\|P_{\T_i^c}q\|_{\op} &\!\!\!\!\!\leq\gamma k_0 \frac{1+\alpha}{1-\alpha}
+ \frac{\bar{\tau}}{k} \frac{k_0}{1-\alpha}.
\end{align*}
\end{lemma}
Note that the previous lemmas provide better bounds that the ones used in the proof of Theorem 2 from \citet{chandrasekaran2011rank}, which allows for the slightly sharper characterization:
\begin{lemma}
\label{eq:valid_gamma_interval}
Let $\alpha:=k_0\sqrt{\frac{2\tauu}{k}}$, if $\alpha+\frac{\alpha^2}{2k_0}<\frac{1}{3}$ then  
$\Gamma:=\big [ \frac{\bar{\tau}}{k}\frac{ 1}{1-3\alpha}, \frac{1}{k_0} \frac{1-k_0 \tauu/k}{1 + \alpha} \big )$ is a non empty interval,
 and for any $\gamma \in \Gamma$, the dual matrix $q$ defined in Lemma~\ref{lem:bounds_projq_help} satisfies conditions \ref{cond:l1_b} and \ref{cond:om_2}.
\end{lemma}

To conclude the proof of Theorem~\ref{theo:two}, note that the assumptions $k_0 \leq \frac{1}{7}\sqrt{k}$ and $\tauu\leq 2$ implies $\alpha+\frac{\alpha^2}{2k_0}<\frac{1}{3}$.
Indeed it implies $\alpha<\frac{2}{7}$ and so $\alpha+\frac{\alpha^2}{2}<\frac{2}{7}+\frac{2}{49}=\frac{16}{49}<\frac{1}{3}.$ As a consequence, Lemmas~\ref{lem:inverse} and~\ref{eq:valid_gamma_interval} apply. The last thing we need to prove is then that $q$ satisfies condition~\ref{cond:om_3}, which  we prove  in Appendix~\ref{app:proof_prop_two} as
\begin{proposition}
\label{prop:end_proof}
Under the assumptions of Theorem~\ref{theo:two},  $\: \forall J \in \mathcal{G}^p_k,\:\lambda_{\max}^+(Q_{\!J\!J})<1.$
\end{proposition}

\subsection{Proof of Theorem \ref{theo:chand}}

Note first that the optimization problem stated in the theorem is equivalent to
$$\min \gamma \|S\|_1+\Omega_p(L) \quad \text{\st} \quad M=S+L,$$
with $\Omega_p$ the gauge associated with the $p$-spsd-rank.

Note that we have just removed the p.s.d. constraint and replaced the trace of $L$ by its trace norm, which should be equivalent if the obtained matrix is p.s.d. 

In order to prove this theorem we need to construct a dual $q\in\text{span}\{\T_0,\T_1,...,\T_r\}$ satisfying \ref{cond:l1_a}, \ref{cond:l1_b}, \ref{cond:om_1}, \ref{cond:om_2} of Proposition~\ref{lem:uniqueopt}. Note that condition \ref{cond:om_3} is void in this context, since we are considering a unique low rank block of rank-one and with full support $\itgset{p}$, and so it it trivially satisfied. But given the assumptions of the theorem, Lemma~\ref{eq:valid_gamma_interval} applies immediately with $k=p$, which yields the result.

\section{Experiments}
\label{experiments}

We first perform experiments on relatively small synthetic graphs and then on a larger one.

\subsection{First experiment}

First, we consider three different LVGGM with $p=45$ observed variables. In each case, we chose the restriction of the graph on observed variables to be a tree (with maximal degree $\leq 5$),
and the graph structure corresponds to latent variables that are independent given all observed variables. The interactions between latent variables and observed variables are chosen as follows :
\begin{itemize}
\item \textit{model 1} has $h=3$ latent variables; we split observed variables in three groups of size $15$ and connect each group to a single latent variable.
\item \textit{model 2}: has $h=3$ latents variables; we split observed variables in three groups of different sizes ($20,15$ and $10$) and connect each group to a single latent variable.
\item \textit{model 3}: has $h=4$ latent variables; we select four overlapping groups of size $15$ with $5$ variables shared between each pair of consecutive groups (see Fig.~\ref{fig:synth}.(b)).
\end{itemize}
The scheme used to construct a sparse precision matrix $K$ for a given graph is described in Appendix~\ref{app:sparse_wishart}. For each mode, we draw $50 p$ random vectors from the corresponding $p$ dimensional multivariate normal distribution and compute the associated marginal empirical covariance matrix from these observations. 

We then estimate the original concentration matrix $K$
 by minimizing the score matching loss regularized either in $\ell_1$-norm and $\Omega$-gauge as in \eqref{opt_nc} or with the $\ell_1$-norm+trace-norm ($\ell_1+\tr$), as proposed by \citet{chandrasekaran2010}. As discussed in Section~\ref{sec:ggm}, for the $\ell_1+\tr$ regularization, the sources are a priori only identified up to a rotation matrix. However, under the assumption that the sources are conditionally independent given observed nodes, $K_{HH}$ is diagonal, and when the groups of observed variables associated with each latent variables are disjoint, the columns of $K_{OH}$ are orthogonal, and are thus proportional to the eigenvectors of $K_{OH}K_{HH}^{-1}K_{HO}$ as soon as the coefficients of the diagonal matrix $K_{HH}$ are all distinct, by uniqueness of the SVD. They are thus identifiable, and it makes sense to estimate the columns of $K_{OH}$ by the eigenvectors of the estimated matrix ${L}$. Obviously, for model 3, we cannot hope to recover $K_{OH}$ with this estimator.

Figure \ref{fig:synth} shows the different estimated concentration matrices obtained, for the choice of hyperparameters $\gamma$ and $\lambda$, that produced matrices $S$ with the correct sparsity level and $L$ with the correct rank.
 
For models 1 and 2, the size of the blocks is fixed. For model 3, we use the gauge $\Omega_w$ introduced in Section~\ref{sec:different_sparsity_levels} which estimates as well the size of the different blocks, based on prior specified via the vector of weights $w$, which penalizes differently different block sizes. We use $w_{k}=\sqrt{k}$ which we found performs reasonably well empirically. The result show clearly that even for models  1 and 3, where, in theory the different columns of $K_{OH}$ could be estimated with an SVD based on the formulation of \citet{chandrasekaran2010}, these columns are not so well estimated and their support would not be estimated correctly by thresholding the absolute value of the estimated coefficients (with perhaps the exception of the smallest component in model 3).

These results show empirically that the proposed formulation performs well beyond the regime for which we provide theoretical guarantees in Section~\ref{sec:id}: first, the experiments are in a finite data setting, so in a sense with noise; then the settings considered are of relatively low dimension with ratio $k_0/k$ and $k_0/\sqrt{k}$ larger than in the theoretical analysis; and we obtained also convincing results for the case where blocks overlap (model 3), or the size of the blocks is estimated as well (model 2).

\begin{figure}
\center
\begin{tabular}{cc}
    \includegraphics[width=.3\linewidth]{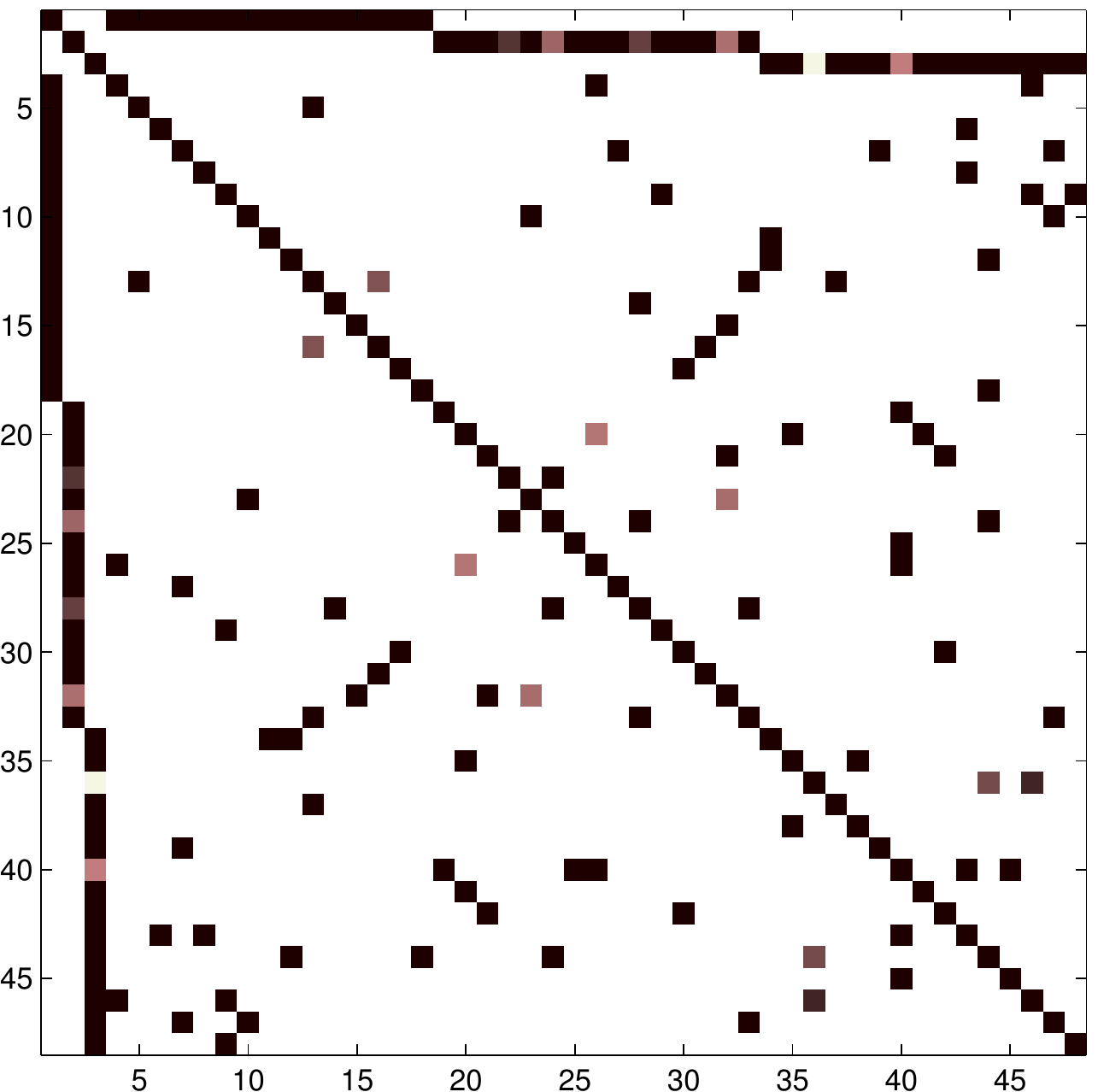} 
  & \includegraphics[width=.3\linewidth]{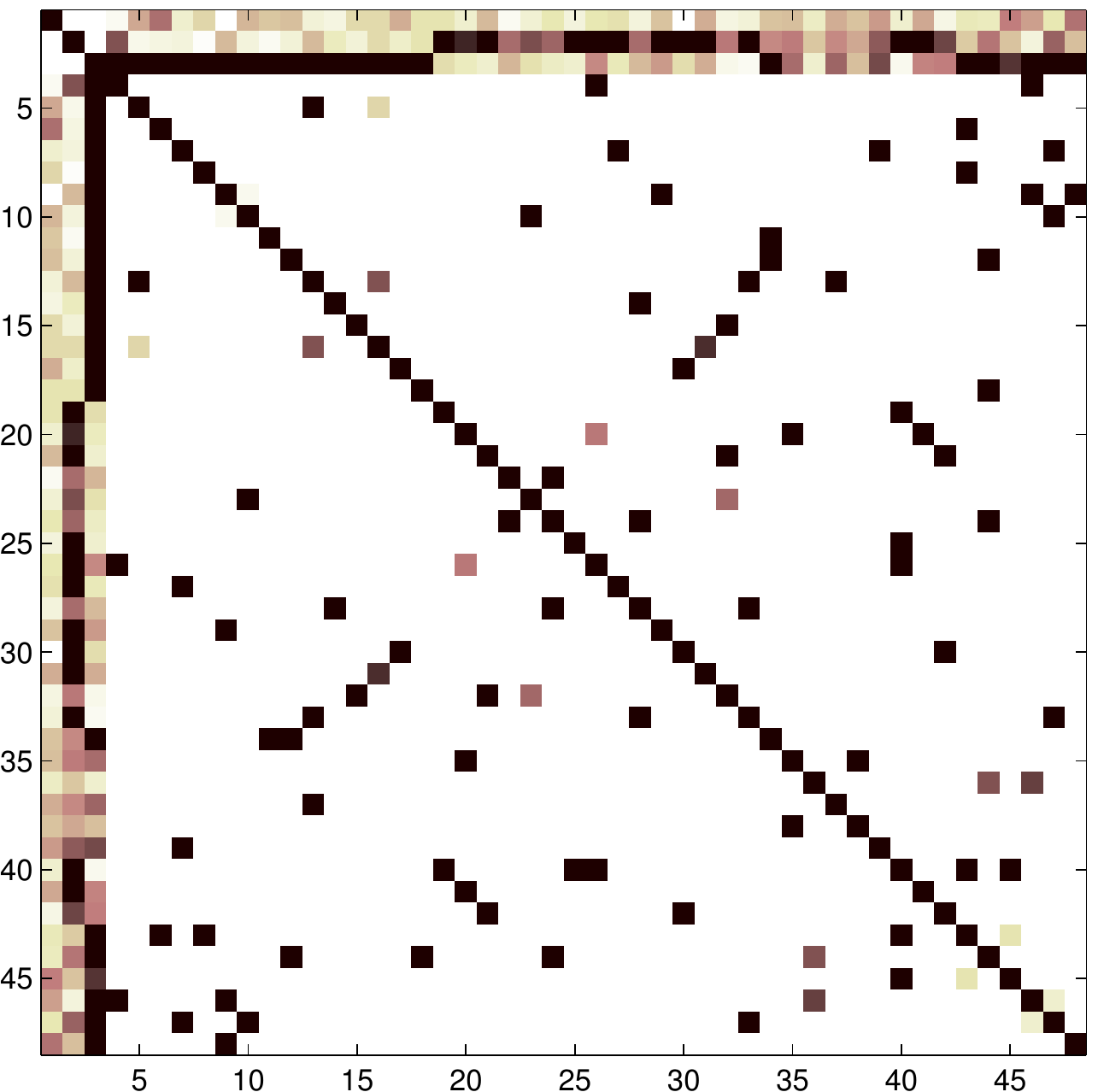} 
   \\    (a) \textit{model 1}, ours & (d)  \textit{model 1}, $\ell_1+\tr$ \\[6pt]
      \includegraphics[width=.3\linewidth]{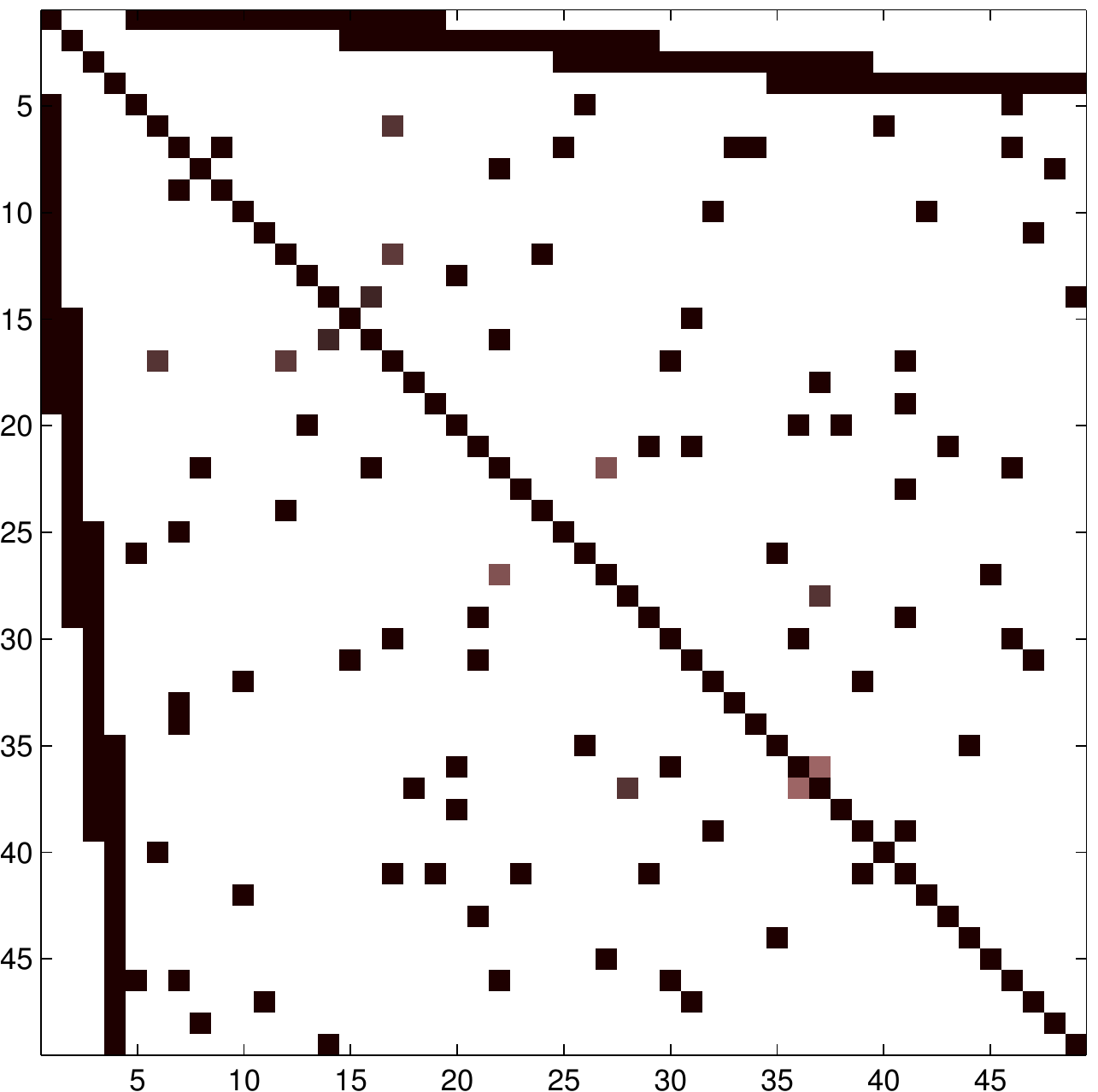}
  &   \includegraphics[width=.3\linewidth]{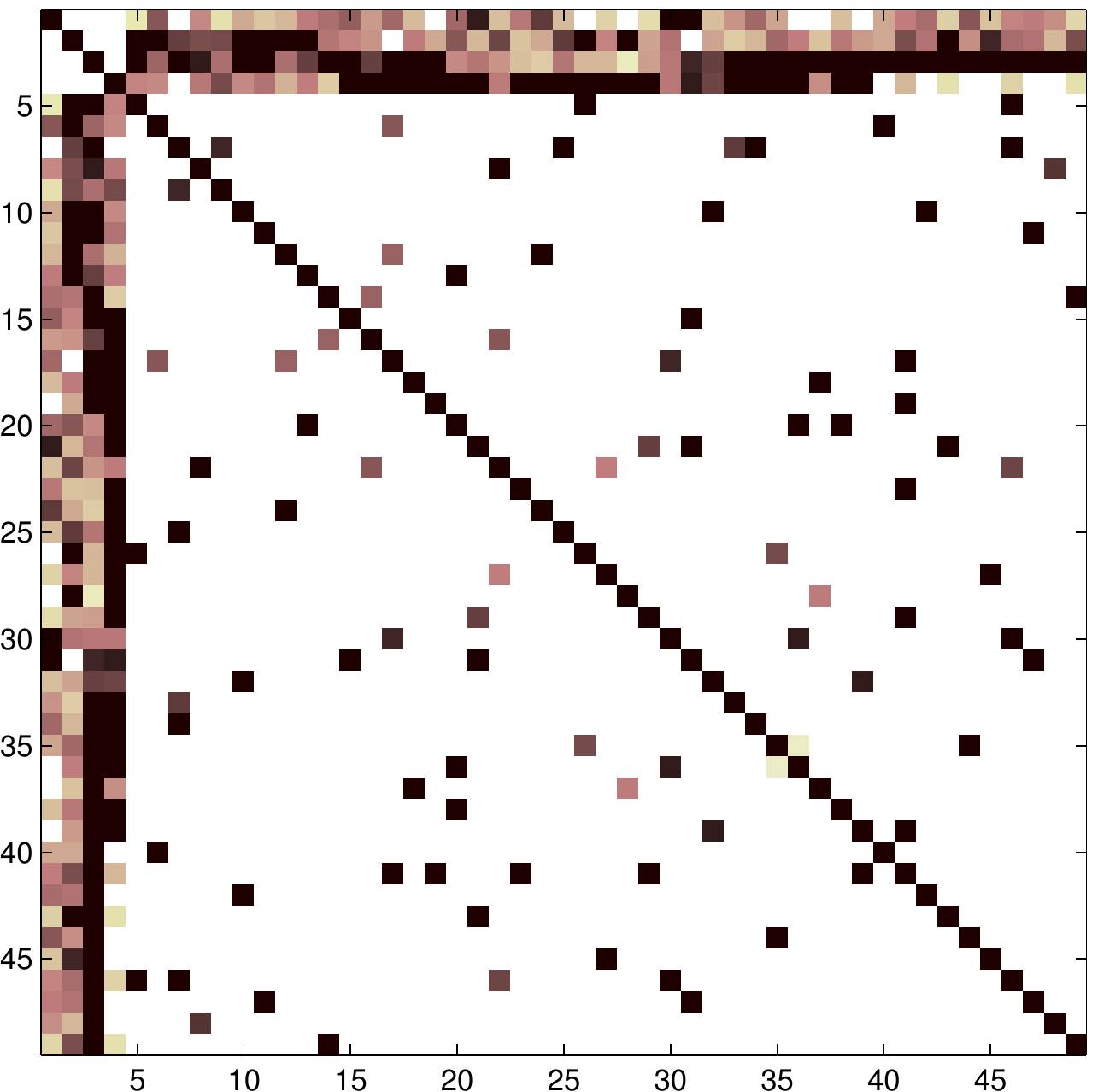}
   \\    (b)  \textit{model 2}, ours   & (e)  \textit{model 2}, $\ell_1+\tr$    \\[6pt]
      \includegraphics[width=.3\linewidth]{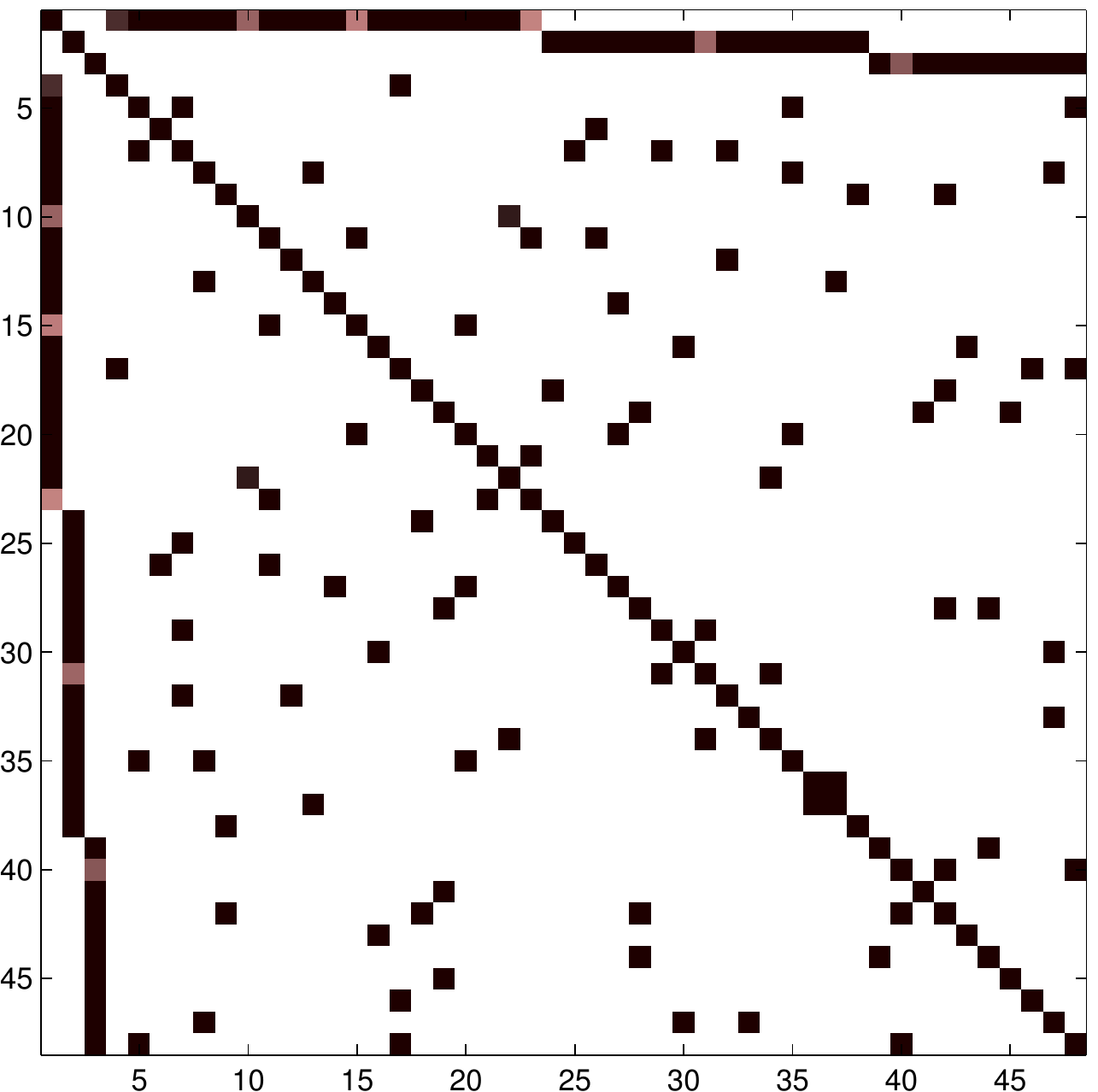}
  &   \includegraphics[width=.3\linewidth]{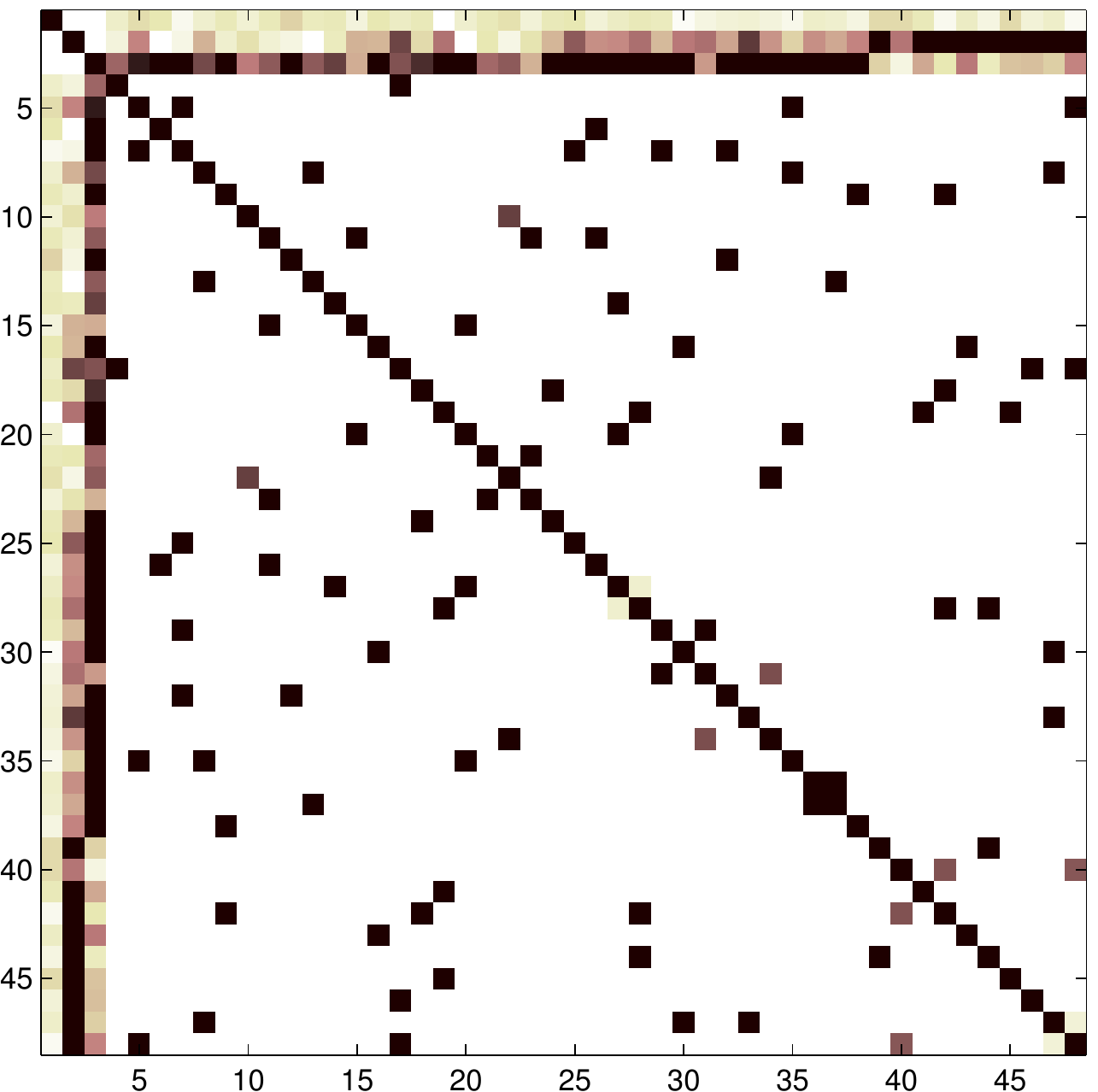}
   \\    (c)  \textit{model 3}, ours & (f)  \textit{model 3}, $\ell_1+\tr$  \\[6pt]
\end{tabular}
\vspace{-1em}
\caption{Estimated $|K_{ij}|$, for $K$ the complete concentration matrices where the three (resp. four) first rows and columns correspond to the latent variables of \textit{model 1} and \textit{model 3} (resp. \textit{model 2}) : for \textit{model 1} in (a) ours and (d) $\ell_1+\tr$ regularization; for \textit{model 2} in (b) ours and (e) $\ell_1+\tr$ regularization; for \textit{model 3} in (c) ours and (f) $\ell_1+\tr$ regularization }
\label{fig:synth}
\end{figure}

\subsection{Second experiment}

We consider a graph which is somewhat larger, with $160$ nodes, corresponding to an empirical covariance matrix which is 12 times larger than the previous ones. 
In this case, the part of the graph corresponding to the observed variables is drawn from an Erd\"os-R\'enyi model, where each edge has a fixed appearance probability $p_{s}=0.01$. We add $4$ latent variables connected to non overlapping groups of $35$ observed variables and we generate $2000$ observations from the full graph. 
We compute the marginal covariance matrix as before (see Appendix~\ref{app:sparse_wishart}) and again solve \eqref{opt_nc} with the score matching loss to compute our estimator.
Figure \ref{fig:synthlarge} shows the low rank component of the ground truth covariance and the low rank component obtained by our method. We clearly recover the latent structure of the graph, i.e., the four groups of $35$ variables.\\

\begin{figure}
\label{fig:synthlarge}
\center
\begin{tabular}{cc}
    \includegraphics[width=.4\linewidth]{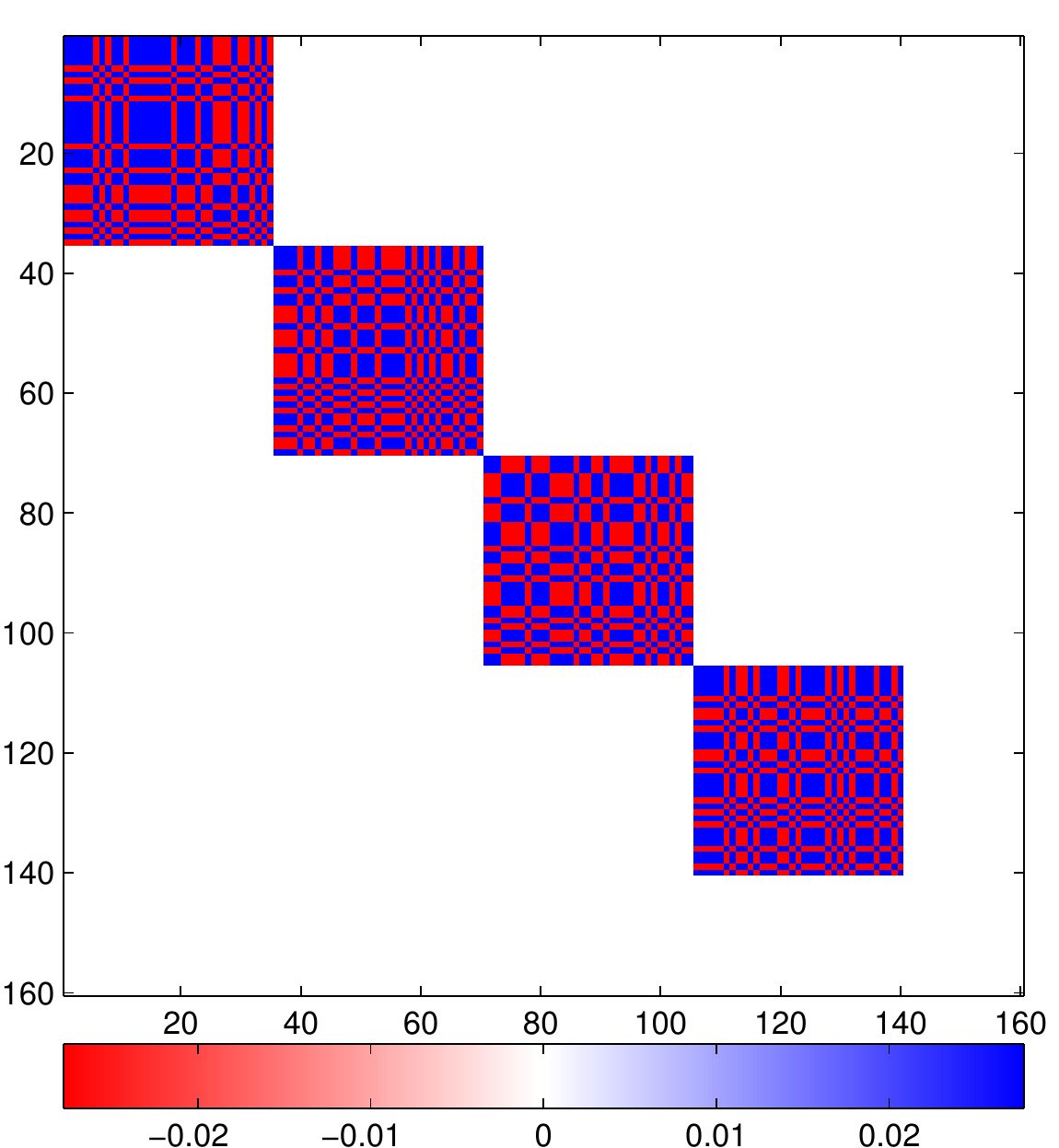} 
  & \includegraphics[width=.4\linewidth]{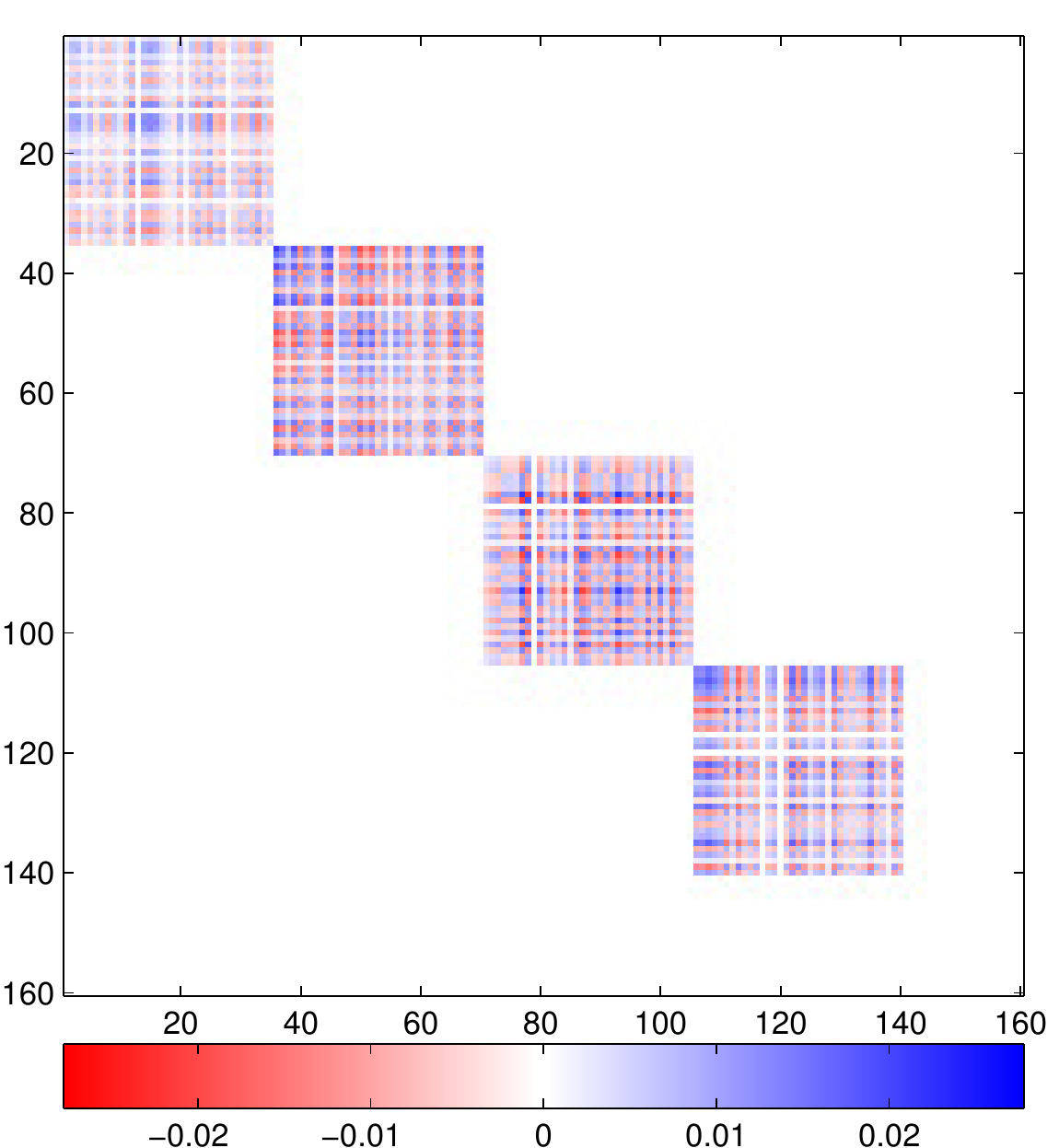} 
   \\    (a) & (b) \\[6pt]
\end{tabular}
\caption{Experiment on on model with $p = 160$ observed variables and 4 unobserved. $n= 2000$ and $k = 35$. (left) low rank component of the ground truth covariance (right) low rank component obtained by our method.}
\end{figure}

\section{Conclusion}
We considered a family of latent variable Gaussian graphical models whose marginal concentration matrix over the observed variables decomposes as a sparse matrix plus a low-rank matrix \emph{with sparse factors}. We introduced a convex regularization to specifically induce this structure on the low rank component, proposed a convex formulation to estimate both components, based on a regularized score matching loss, and proposed an efficient algorithm to solve it. We provided as well an identifiability result, that guarantees that, in the limit of an infinite amount of data, and when the blocks associated with each latent variable are disjoint, the graph structure of the whole graph, including connectivity between latent and observed variables is recovered by the proposed formulation. 

Our experiments show promising results in terms of recovery of the structure of the whole graph, including when there is overlap or when cliques associated with latent variables have different sizes. Future work could study more precisely the formulations that allows for different clique sizes, and extend identifiability/recovery results in different directions.

\bibliographystyle{plainnat}
\bibliography{latentggm_ready}

\newpage
\appendix
\section{Supplementary material}

\subsection*{Proof of Lemma~\ref{lem:LMO}}
\begin{claim}
Let $Y\in\RR^{p\times p}$ be a symmetric matrix. The polar gauge of $\Omega$ writes
\begin{align}
{\Omega^{\circ}}\!(Y)= \max_{I\in\mathcal{G}^p_k}\lambda^{+}_{\max}(Y_{II}).
\end{align}
\end{claim}
\begin{proof}
$\displaystyle \Omega^{\circ}(Y) 
=\max_{\Omega(X)\leq 1} {\tr}{(Y^{\top}X)}
=\max_{\substack{\|u\|_0=k \\ \|u\|_ 2=1 }} u^{\top}Yu
=\max_{I\in\mathcal{G}^p_k}\lambda^{+}_{\max}(Y_{II}).$
\end{proof}

\subsection*{Lemmas charactering the subgradients}
In the following lemmas we express the subgradients of the $\ell_1$ norm and $\Omega$ as decomposed on the tangent subspaces. The result for the $\ell_1$-norm is well known.
\begin{lemma}
(Characterization of $\ell_1$ subgradient) $Q\in \gamma \partial \|.\|_1(S^*)$ if and only if
\begin{enumerate}[label=\textbf{(A.\arabic*)}, leftmargin=2\parindent]
\item \label{cond:l1_abis} $\mathcal{P}_{\T_0}(Q)=\gamma \sign(S^*)$
\item \label{cond:l1_bbis} $\|\mathcal{P}_{\T^c_0}(Q)\|_{\infty}\leq \gamma$
\end{enumerate}
\end{lemma}
We then characterize the subgradient of the gauge we have introduced.
\begin{lemma}
(Characterization of the subgradient of $\Omega$)\\ 
If $L^*$ is of the form $L^*=\sum_{i=1}^r s_i u^i {u^i}^\top$, with $\supp(u^i)\subset I_i$ and $I_i \cap I_j =\varnothing$ for all $i\neq j$, we have that $Q\in \partial \Omega(L^*)$ if and only if
\begin{enumerate}[label=\textbf{(B.\arabic*)}, leftmargin=2\parindent]
\item \label{cond:om_1bis} $\forall i \in \itgset{r},\:\mathcal{P}_{\T_i}(Q)=u^i{u^i}^\top$
\item \label{cond:om_2bis} $\forall i \in \itgset{r},\:\lambda_{\max}^+\big (\mathcal{P}_{\T_i^c}(Q) \big )\leq 1$
\item \label{cond:om_3bis} $\forall J \in \mathcal{G}^p_k \backslash \{I_1,\ldots, I_r\}, \: \lambda_{\max}^+(Q_{JJ})\leq 1$
\end{enumerate}
\end{lemma}
\begin{proof}
By the characterization of the subgradient of a gauge we have  $Q\in \partial \Omega(L^*)$ if and only if
\begin{align}
\max_{I \in \mathcal{G}^p_k}\lambda_{\max}^+(Q_{II})\leq 1\quad \text{and} \quad \left\langle Q,L^*\right\rangle=\Omega(L^*). \label{eq:subgom}
\end{align}
The inequality implies immediately \ref{cond:om_3} and that $u^{\top}\!Qu\leq1$ for any unit vector $u$ such that $\|u\|_0\leq k$. By definition of $L^*$, the equality becomes $\sum_{I_i\in\mathcal{I}}s_i ({u^i}^{\top}Qu^i-1)=0$. Since all terms of the sum are non negative we must have ${u^i}^{\top}Qu^i=1$. 
Since $1={u^i}^{\top}Qu^i={u^i}^{\top}Q_{I_iI_i}u^i$ and we have $\lambda_{\max}^+(Q_{I_iI_i})\leq 1$, $u^i$ must be an eigenvector of $Q_{I_iI_i}$ with eigenvalue $1$. Given that $Q_{I_iI_i}$ as a real symmetric matrix, admits an orthonormal basis of eigenvectors, we can thus write $Q_{I_iI_i}=u^i{u^i}^{\top} + W_i$ with $W_i \in \T_i^c$ and $\lambda_{\max}^+(W_i)\leq 1$. Since the previous decomposition shows that $W_i=\P_{\T_i^c}(Q)$ and $\P_{\T_i}(Q)=u^i{u^i}^{\top}$ we have shown \ref{cond:om_1} and \ref{cond:om_2}.
\end{proof}

\subsection*{Proof of Proposition~\ref{lem:uniqueopt}}
\begin{claim}
The pair $(S^*,L^*)$  is the unique optimum of \eqref{pb:main} if
\begin{enumerate}[label=\textbf{(T)}, leftmargin=4\parindent]
\item \label{cond:transversality} $\forall i \in \itgset{r},\quad \T_0\cap \T_i = \{0\}$, 
\end{enumerate}
and there exists a dual matrix $Q\in\RR^{p\times p}$ such that:
\begin{enumerate}[label=\textbf{(S.\arabic*)}, leftmargin=4\parindent]
\item \label{cond:l1_a} $\mathcal{P}_{\T_0}(Q)=\gamma \, \sign(S^*)$
\item \label{cond:l1_b} $\|\mathcal{P}_{\T^c_0}(Q)\|_{\infty}<\gamma$
\end{enumerate}
\vspace{-1.7em}
\begin{enumerate}[label=\textbf{(L.\arabic*)}, leftmargin=4\parindent]
\item \label{cond:om_1} $\forall i \in \itgset{r},\quad\mathcal{P}_{\T_i}(Q)=u^i{u^i}^\top$
\item \label{cond:om_2} $\forall i \in \itgset{r},\quad \lambda_{\max}^+ \big ( \mathcal{P}_{\T_i^c}(Q) \big )<1$
\item \label{cond:om_3} $\forall J \in \mathcal{G}^p_k \backslash \{I_1,\ldots, I_r\}, \quad \lambda_{\max}^+ (Q_{JJ})< 1.$
\end{enumerate}
\end{claim}
\begin{proof}

The \ref{cond:l1_a}, \ref{cond:l1_b}, \ref{cond:om_1}, \ref{cond:om_2} and \ref{cond:om_3} clearly imply that there exist a dual matrix $Q$ such that
$Q \in \big (\gamma \partial\|\cdot\|_1(S^* )\big ) \cap \partial\Omega(L^* )$, which is the first order subgradient condition that characterizes the optima of~\eqref{pb:main}.

To show that the solution is \emph{unique} we show that $(S^*,L^*)$ must be obtained as the unique solution of an equivalent minimization problem.
Indeed, consider the gauge $\gamma_I(M)=\tr(M)+\iota_{\{M \succeq 0\}}+\iota_{\{\supp(M) \subset I \times I\}}.$ It is immediate to verify that the polar gauge is $\gamma_I^{\circ}$ such that $\gamma_I^{\circ}(Q)=\lambda_{\max}^+(Q_{II}).$ Thus $\Omega^{\circ}(Q)=\max_{I \in \mathcal{G}^p_k} \gamma_I^{\circ}(Q)$ and, taking polars, we get that
\begin{equation}
\label{eq:omega_conv_inf}
\Omega(M)=\inf \Big \{\sum_{I \in \mathcal{G}^p_k} \gamma_I(\Mi) \mid M=\sum_{I \in \mathcal{G}^p_k}\Mi\Big \}.
\end{equation}

As a consequence, problem~\eqref{pb:main} is equivalent to 
\begin{equation}
\label{pb:decomposed}
\min_{S,(\Li)_{I \in \mathcal{G}^p_k}} \gamma\|S\|_1+\sum_{I \in \mathcal{G}^p_k} \gamma_I(\Li) \quad \st \quad M=S+\sum_{I \in \mathcal{G}^p_k} \Li.
\end{equation}
In particular, if $\big (S^*,(\Lis)_{I \in \mathcal{G}^p_k}\big )$ is an optimal solution of \eqref{pb:decomposed}, and if $L^*=\sum_{I \in \mathcal{G}^p_k} \Lis,$ then $(S^*,L^*)$ is an optimal solution of~\eqref{pb:main}. Conversely, $(S^*,L^*)$ is an optimal solution of~\eqref{pb:main}, then any optimal decomposition of $L^*$ obtained from \eqref{eq:omega_conv_inf} yields an optimal solution of \eqref{pb:decomposed}.

So clearly, if the solution to \eqref{pb:decomposed} is unique, then so must be that of~\eqref{pb:main}. 

Let's then assume that $\big (S^*+N_0,(\Lis+\Ni)_{I \in \mathcal{G}^p_k}\big )$ is another optimal solution to \eqref{pb:decomposed}. Since matrices in both solutions sum to $M$, we must necessarily have 
\begin{equation}
\label{eq:null_sum}
N_0+\sum_{I \in \mathcal{G}^p_k} \Ni=0.
\end{equation}
Let $\Qi \in \partial \gamma_{I}(\Lis)$ and $Q_0 \in \partial\|\cdot\|_1(S^*)$. Then, by convexity, we have
\begin{align}
\label{eq:neg}
\gamma \|S^*\|_1+\sum_{I \in \mathcal{G}^p_k} \gamma_I(\Lis) & =\gamma \|S^*+N_0\|_1+\sum_{I \in \mathcal{G}^p_k} \gamma_I(\Lis+\Ni)\\
&\geq \gamma \|S^*\|_1+\sum_{I \in \mathcal{G}^p_k} \gamma_I(\Lis)+\dotp{Q_0}{N_0}+\sum_{I\in \mathcal{G}^p_k} \dotp{\Qi}{\Ni}. \notag
\end{align}
Consistently with previous notations, we denote by $\I=\{I_i,\ldots,I_r\}$ the set of blocks such that $\Lis\neq 0$, and $Q_i:=Q_{I_i}$, $N_i:=N_{I_i}$.

Now, $\gamma_I$ is a decomposable gauge in the sense of \citet{negahban2012unified}:  in particular if $L^{\scriptscriptstyle (I_i)\, {\scriptstyle*}}=L_i^*:=U^i D^i {U^i}^\top$, with $U^i$ an orthonormal matrix and $D^i$ a diagonal matrix, then $\partial \gamma_{I_i}(L_i^*)=\big \{Q_i^*+ Q_i^c \mid Q_i^c \in \T_i^c,\: \gamma_{I_i}^{\circ}(Q_i^c)\leq 1\big\},$ with $Q_i^*=U^i {U^i}^\top\!.$ Note that, since $\T_i$ and $\T_i^c$ are orthogonal, for all $i \in \itgset{r},$ any  $Q_i \in \gamma_{I_i}(L_i^*)$ is such that $\P_{\T_i}(Q_i)=Q_i^*$. In the rest, of the proof, we choose $Q_i=Q^*_i+Q_i^c$ with $Q_i^c \in \T_i^c$ such that 
\begin{equation}
\label{eq:Qi_polar_to_Ni}
\gamma_{I_i}\big (\P_{\T_i^c}(N_i) \big )=\dotp{\P_{\T_i^c}(N_i)}{Q_i^c}=\dotp{\P_{\T_i^c}(N_i)}{Q_i}
\end{equation}
 (this is clearly possible because for $M \in \T_i^c$, we have precisely that $\gamma_{I_i}(M)=\max\{\dotp{M}{Z} \mid {Z\in \T_i^c,\: \gamma^{\circ}_{I_i}(Z)\leq 1}\}$).

Given that there exists, by assumption of the theorem, $Q$ such that conditions \ref{cond:l1_a},\ref{cond:l1_b},\ref{cond:om_1},\ref{cond:om_2},\ref{cond:om_3} are satisfied, we have in particular that $\P_{\T_i}(Q)=Q_i^*,\: \forall i \in \{0\}\cup \itgset{r},$ with $Q_0^*=\gamma \sign(S^*)$.
 
 So, we have
\begin{align*}
0 &\overset{\eqref{eq:neg}}{\geq} \dotp{Q_0}{N_0}+\sum_{I\in \mathcal{G}^p_k} \dotp{\Qi}{\Ni} \\
&= \sum_{i=0}^r \big (\dotp{Q_i^*}{N_i}+\dotp{\P_{\T_i^c}(Q_i)}{N_i} \big )+\sum_{I \in \mathcal{G}^p_k \backslash{\I}}\dotp{\Qi}{\Ni}\\
&= \sum_{i=0}^r \big (\dotp{Q}{N_i}+\dotp{\P_{\T_i^c}(Q_i-Q)}{N_i} \big )+\sum_{I \in \mathcal{G}^p_k \backslash{\I}}\dotp{\Qi}{\Ni}\\
&\overset{\eqref{eq:null_sum}}{=} \sum_{i=0}^r \dotp{Q_i-Q}{\P_{\T_i^c}(N_i)}+\sum_{I \in \mathcal{G}^p_k \backslash{\I}}\dotp{\Qi-Q}{\Ni}\\
& \overset{\eqref{eq:Qi_polar_to_Ni}}\geq \gamma \|\P_{\T_0^c}(N_0)\|_1 \big (1-{\textstyle \frac{1}{\gamma}} \|\P_{\T_0^c}(Q)\|_{\infty} \big )+\sum_{i=1}^r \gamma_{I_i}\big (\P_{\T_i^c}(N_i) \big )\big (1-\gamma_{I_i}^{\circ}\big (\P_{\T_i^c}(Q) \big )\big )\\
& \hspace{4.85cm}+\sum_{I \in \mathcal{G}^p_k\backslash \I}  \gamma_{I}(\Ni) \big (1-\gamma_{I}^{\circ}(Q) \big ),
\end{align*}
where the last inequality is an instance of the Fenchel-Young inequality.
But this last expression is non negative and, as a consequence of conditions \ref{cond:l1_b},\ref{cond:om_2} and \ref{cond:om_3}, can only be equal to zero if, 
$$
\begin{cases}
 &\|\P_{\T_0^c}(N_0)\|_1=0,\\
\forall i \in \itgset{r},\quad &\gamma_{I_i}\big (\P_{\T_i^c}(N_i) \big )=0,\\
\forall I\in \mathcal{G}^p_k\backslash \I, \quad &\gamma_{I}(\Ni)=0.
\end{cases}
$$ 

So $\forall I \notin \I,\: \Ni=0$, and for all $0\leq i \leq r, \: N_i \in \T_i.$ Finally by \eqref{eq:null_sum}, we have $\sum_{i=0}^r N_i=0$, and by projecting this equality on $\bar{\T}_{i}$ we get 
$N_{0,i}+N_i=0$ with $N_{0,i}:=\P_{\bar{\T}_{i}}(N_0) \in \T_0$ and $N_i \in \T_i$. But, by \ref{cond:transversality}, $\T_0 \cap \T_i=\{0\}$, i.e. the two spaces are in direct sum, in which case the fact that $N_{0,i}+N_i=0$ implies $N_{0,i}=0$ and $N_i=0$. We clearly have $N_0=\P_{\bar{\T}_{00}}(N_0)+\sum_{i=1}^r N_{i,0}=0$, since $\P_{\bar{\T}_{00}}(N_0)=0$ by projection of $\sum_{i=0}^r N_i=0$ on $\bar{\T}_{00}$.
And so finally, for all $0\leq i \leq r,\: N_i=0$, which shows that the solution is necessarily unique.
\end{proof}

\subsection*{Proof of Lemma~\ref{lem:bounds_xxi}}
\begin{claim}
(Bounds on $\xxi$)
Let us consider the elements of Definition~\ref{def:xxi}. Given the definitions of $k_0$ and $\tauu$, we have
\begin{enumerate}[label={(\arabic*}), leftmargin=2\parindent]
\item \label{cond:xxi_0i} $\xxi_{i\rightarrow 0} \leq  \sqrt{\frac{2\bar{\tau}}{k}}$
\item \label{cond:xxi_i0} $\xxi_{0\rightarrow i} \leq k_0 $
\item \label{cond:xxi_0ip} $\xxi_{i\rightarrow 0}' \leq \sqrt{\frac{2\bar{\tau}}{k}}$
\item \label{cond:xxi_i0p} $\xxi_{0\rightarrow i}' \leq 2 k_0 \sqrt{\frac{k_0 \bar{\tau}}{k}}$ 
\end{enumerate}
\end{claim}
\begin{proof}
\ref{cond:xxi_0i} Let $M$ be any matrix in $\T_i$ such that  $\|M\|_{\op}\leq 1$. We know that $\exists v$ with $\supp(v) \subset I_i$ such that $M=u^iv^{\top}+v{u^i}^{\top}$.  The condition  $\|M\|_{\op}\leq 1$ imposes in particular $|{u^i}^{\top}Mv/\|v\||\leq 1$ which becomes $\|u^i\|^2\|v\| \leq 1 - ({u^i}^{\top}v)^2/\|v\|$. Hence $\|v\|\leq 1$, and
\begin{align*}
\|M\|_{\infty} &= \|u^iv^{\top}+v{u^i}^{\top}\|_{\infty} \\
&\leq \max_{k,l} \rule{0mm}{1.5pc} \big [ |u^i_k||v_l| + |u^i_l||v_k| \big ] \\
&\leq \|u^i\|_{\infty}\max_{k,l} \big [ |v_l| + |v_k| \big ] 
\leq \|u^i\|_{\infty}\sqrt{2} \sqrt{ v_l^2 + v_k^2} 
\leq \sqrt{\frac{2\bar{\tau}}{k}},
\end{align*}
since $\|u^i\|_{\infty}^2\leq \frac{\bar{\tau}}{k}$.

\ref{cond:xxi_0ip} Since $\|\mathcal{P}_{\T_0}(M)\|_{\infty} \leq\|M\|_{\infty}$, we have $\xxi'_{i\rightarrow 0}\leq \xxi_{i\rightarrow 0}$ .\\ 
For the other two inequalities, let $Z$ be any matrix in $\T_0$ such that  $\|Z\|_{\infty}\leq 1$. Then we know that $\supp(Z)\subset \supp(S^*)$. Let us introduce variables $\delta$ such that $\delta_{ij}= 1$ if $S^*_{ij}\neq 0$ and $\delta_{ij}= 0$ otherwise. We notice that, for any $v\in \RR^{p}$,
\begin{align}
\hspace{-5mm}\|Zv\|_2 &= \max_{w: \|w\|_2\leq 1} |w^{\top}Zv|  \nonumber \\
 & =  \max_{w: \|w\|_2\leq 1} \sum_{i,j} |v_i| |w_j| |Z_{ij}| \nonumber \\
 & \leq  \|Z\|_{\infty} \max_{w: \|w\|_2\leq 1} \sum_{i,j} |v_i| |w_j|\delta_{ij} \nonumber \\
&\leq  \|Z\|_{\infty}   \max_{w: \|w\|_2\leq 1} \sqrt{\sum_{i,j} \delta_{ij}}  \sqrt{\sum_{i,j}v_i^2 w_j^2 \delta_{ij}}
\leq  \|Z\|_{\infty} \sqrt{\|Z\|_0}  \|v\|_2, \label{eq:bound_op}
\end{align}
where the second inequality uses Cauchy-Schwarz and the last inequality uses the fact that $\sum_{i,j} \delta_{ij}= \|Z\|_0\leq k_0^2$ and the fact that $|\delta_{ij}|\leq 1$.\\
It follows immediately from \eqref{eq:bound_op} that
\begin{equation}
\label{eq:cauchy_schwarz_op}
\|Z\|_{\op} \leq \|Z\|_{\infty} \sqrt{\|Z\|_0}.
\end{equation}

Inequality \ref{cond:xxi_i0} follows from \eqref{eq:cauchy_schwarz_op} and the fact that $\|Z\|_0^2.$

To prove \ref{cond:xxi_i0p}, note that  since $\mathcal{P}_{\T_i}(Z)=u^i{u^i}^{\top}Z - u^i{u^i}^{\top}Zu^i{u^i}^{\top}+Zu^i{u^i}^{\top}$,
\begin{align*}
\|\mathcal{P}_{\T_i}(Z)\|_{\op} &=  \|u^i{u^i}^{\top}Z(I- u^i{u^i}^{\top}\!)\|_{\op}+\|Zu^i{u^i}^{\top}\!\|_{\op} 
\leq 2 \|Zu^i\|_2.
\end{align*}
But then using the same derivation as the one leading to \eqref{eq:bound_op}, we have
\begin{align*}
\|Zu^i\|_2 &\leq  \|Z\|_{\infty}   \max_{w: \|w\|_2\leq 1} \sqrt{\sum_{j,j'} \delta_{jj'}}  \sqrt{\sum_{j,j'}{u^i_j}^2 w_{j'}^2 \delta_{jj'}} \\
&\leq \|Z\|_{\infty} \, k_0 \, \|u^i\|_{\infty} \sqrt{\sum_{j'} w_{j'}^2 \big (\sum_{j}\delta_{jj'}\big )}\leq 2 \|Z\|_{\infty} \, k_0 \sqrt{\frac{k_0 \bar{\tau}}{k}}.
\end{align*}
\end{proof}

\subsection*{Proof of Lemma~\ref{lem:transversality}}
\begin{claim}[Transversality condition]
Let $\alpha:=k_0\sqrt{\frac{2\tauu}{k}}$. If $\alpha<1,$ then, for all $i \in\itgset{r}$, $\T_0\cap \T_i = \{0\}$.
\end{claim}
\begin{proof}
Let $M\in\T_0\cap \T_i$, then by definition of $\xxi_{0\rightarrow i}$ and $\xxi_{i\rightarrow 0}$ we have 
$$
\|M\|_\infty=\|\mathcal{P}_{\T_0}\circ \mathcal{P}_{\T_i}(M)\|_\infty \leq \xxi_{i\rightarrow 0}\,\xxi_{0\rightarrow i} \|M\|_\infty.
$$
Hence, if $\xxi_{i\rightarrow 0}\,\xxi_{0\rightarrow i}<1$ the only possible solution is $M=0$. But given the upper bounds on $\xxi_{i\rightarrow 0}$ and $\xxi_{0\rightarrow i}$ established in Lemma~\ref{lem:bounds_xxi} we get the result as soon as $\sqrt{\frac{2\bar{\tau}}{k}}k_0<1$.
\end{proof}

\section{Technical lemmas from the proof of Theorem~\ref{theo:two}}

\subsection*{Proof of Lemma~\ref{lem:inverse}}
\begin{claim} 
Let $A:=\begin{bmatrix}
I & P_{\T_0}\\
P_{\T_i} & I\\
\end{bmatrix}$. Then, with Definition~\ref{def:xxi}, if $(1- \xxi_{0\rightarrow i}\,\xxi_{i\rightarrow 0})>0$, then $A$ is invertible and its inverse is 
$$
B:=
\begin{bmatrix}
I & -P_{\T_0}\\
-P_{\T_i} & I\\
\end{bmatrix}
\begin{bmatrix}
(I-P_{\T_0}P_{\T_i})^{-1} & 0\\
0 & (I-P_{\T_i}P_{\T_0})^{-1}\\
\end{bmatrix}.
$$
\end{claim}
\begin{proof}
Clearly, $AB=I$. We need to show that $(I-P_{\T_0}P_{\T_i})$ and $(I-P_{\T_i}P_{\T_0})$ are invertible. Let $x$ be any matrix in $\RR^{p\times p}$. From Definition~\ref{def:xxi}, we have
\begin{align*}
\|(I-P_{\T_0}P_{\T_i})x\|_{\infty} &\geq \|x\|_{\infty}-\|P_{\T_0}P_{\T_i}x\|_{\infty}\\ 
&\geq \|x\|_{\infty}-\xxi_{i\rightarrow 0}\|P_{\T_i}x\|_{\op} 
\geq \|x\|_{\infty}-\xxi_{i\rightarrow 0}\,\xxi_{0\rightarrow i}\|x\|_{\infty}.
\end{align*}
Hence, if $x\neq 0$, $\|(I-P_{\T_0}P_{\T_i})x\|_{\infty}\geq (1-\alpha)\|x\|_{\infty}>0$ which shows that $(I-P_{\T_0}P_{\T_i})$ is invertible. Moreover if we let $x=(I-P_{\T_0}P_{\T_i})^{-1}v$ in this inequality, we get the last inequality at the end of the theorem. The case of $I-P_{\T_i}P_{\T_0}$ is exactly symmetric.
\end{proof}

\subsection*{Proof of Lemma~\ref{lem:bounds_projq_help}}

\begin{claim}(Bounds on $\|P_{\T_0^c}q\|_{\infty}$ and $\|P_{\T_i^c}q\|_{\op}$)
\begin{align*}
\|P_{\T_0^c}q\|_{\infty} 
&\leq \max_{i\in\itgset{r}}\|q_i^*\|_\infty + \xxi_{i\rightarrow 0}\|\varepsilon_i \|_{\op}, \\
\|P_{\T_i^c}q\|_{\op} 
&\leq  \|q_0^*\|_{\op} + \xxi_{0\rightarrow i}\|\varepsilon_0 \|_{\infty}
\end{align*}
\end{claim}

\begin{proof}
By Equation~\eqref{eq:q_decomp},
\begin{align*}
\|P_{\T_0^c}q\|_{\infty} &= \big\|\sum_{i=1}^r P_{\T_0^c}q_i^* + P_{\T_0^c}\varepsilon_i \big\|_{\infty} \\
&\leq \max_{i\in\itgset{r}}\|P_{\T_0^c}q_i^* + P_{\T_0^c}\varepsilon_i \|_{\infty}\\
&\leq \max_{i\in\itgset{r}}\|q_i^* + \varepsilon_i \|_{\infty}
\leq \max_{i\in\itgset{r}}\big (\|q_i^*\|_\infty + \|\varepsilon_i \|_{\infty} \big )
\leq \max_{i\in\itgset{r}}\big ( \|q_i^*\|_\infty + \xxi_{i\rightarrow 0}\|\varepsilon_i \|_{\op}\big ),
\end{align*}
where the first inequality is due to the fact that for each $i\in\{1,\ldots, r\}$, $P_{\T_0^c}q_i^* + P_{\T_0^c}\varepsilon_i$ has its support in $I_i\times I_i$ and $I_i$ are disjoint. The second inequality comes from the fact that for any matrix $A$, $\|P_{\T_0^c}A\|_{\infty} = \max_{i,j\notin \text{supp}(S^*)} |A_{ij}| \leq  \|A\|_{\infty}$. \\
For $\|P_{\T_i^c}q\|_{\op}$, we have
\begin{align*}
\|P_{\T_i^c}q\|_{\op}
&\leq  \|q \|_{\op}
\leq  \|q_0^* + \varepsilon_0 \|_{\op}
\leq  \|q_0^*\|_{\op} + \|\varepsilon_0 \|_{\op}
\leq  \|q_0^*\|_{\op} + \xxi_{0\rightarrow i}\|\varepsilon_0 \|_{\infty},
\end{align*}
where the first inequality is due to the fact that for any matrix $Z$, $$\|P_{\T_i^c}Z\|_{\op} = \|(I-{u^i u^i}^{\top})Z(I-{u^i u^i}^{\top})\|_{\op}\leq \|Z\|_{\op}.$$ 
\end{proof}

\subsection*{Proof of Lemma~\ref{lem:bounds_eps}}
\begin{claim}(Bounds on $\varepsilon_i$)
If $\xxi_{0\rightarrow i}\,\xxi_{i\rightarrow 0} \leq \alpha<1,$ and $(\varepsilon_i)_{i\in\itgset{r}}$ be defined as in the previous lemma, then
\begin{align*}
&\|\varepsilon_0\|_{\infty} \leq {\textstyle \frac{1}{1-\alpha}}\big(\frac{\bar{\tau}}{k} + \xxi_{i\rightarrow 0}'2\gamma k_0\big) \qquad \text{and}\qquad 
\|\varepsilon_i\|_{\op} \leq  {\textstyle \frac{1}{1-\alpha}}\big(2\gamma k_0 + \xxi_{0\rightarrow i}'\frac{\bar{\tau}}{k}\big).
\end{align*}
\end{claim}
\begin{proof}
By Lemma~\ref{lem:inverse}, we have

\begin{equation}
\begin{bmatrix}
\varepsilon_{0,i} \\
\varepsilon_i \\
\end{bmatrix} 
=
\begin{bmatrix}
I & -P_{\T_0}\\
-P_{\T_i} & I\\
\end{bmatrix}
\begin{bmatrix}
(I-P_{\T_0}P_{\T_i})^{-1} & 0\\
0 & (I-P_{\T_i}P_{\T_0})^{-1}\\
\end{bmatrix}
\begin{bmatrix}
\eta_0\\
\eta_i\\
\end{bmatrix}
\end{equation}

So, if, for $i\in\itgset{r}$, we let $\tilde{\eta}_{0,i}:=(I-P_{\T_0}P_{\T_i})^{-1}\eta_0$ and $\tilde{\eta}_{i}:=(I-P_{\T_i}P_{\T_0})^{-1}\eta_i$, then $\varepsilon_0,\ldots, \varepsilon_r$ are uniquely defined by
\begin{align*}
\left\{
    \begin{array}{ll}
    	\varepsilon_{0} = \sum_{i=1}^r \varepsilon_{0,i} \quad \text{where} \enskip \varepsilon_{0,i} = \tilde{\eta}_{0,i} -  P_{\T_i}\tilde{\eta}_{i}, 
    	\vspace{1em}\\
        \varepsilon_i = \tilde{\eta}_{i} -  P_{\T_0}\tilde{\eta}_{0,i} \quad \text{for} \enskip i\in\itgset{r}.	
    \end{array}
\right.
\end{align*}
In the rest of the proof, we use the fact that $\xxi_{0\rightarrow i}\,\xxi_{i\rightarrow 0}\leq \alpha$. First, using the inequalities proved in Lemma~\ref{lem:inverse}, we have, for $i\geq 1$, 
$$\|\tilde{\eta}_{0,i}\|_{\infty} \leq {\textstyle \frac{1}{1-\alpha}} \|\eta_0\|_{\infty} \quad  \text{and} \quad  \|\tilde{\eta}_{0,i}\|_{\op}\leq {\textstyle \frac{1}{1-\alpha}} \|\eta_i\|_{\op}.$$ 

Then, we can bound $\|\varepsilon_{0,i}\|_{\infty}$ as follows
\begin{align*}
\|\varepsilon_{0,i}\|_{\infty} &= \|\tilde{\eta}_{0,i} - P_{\T_0}\tilde{\eta}_i\|_{\infty}\\
&\leq \|\tilde{\eta}_{0,i}\|_{\infty} +\| P_{\T_0}\tilde{\eta}_i\|_{\infty}
\leq \|\tilde{\eta}_{0,i}\|_{\infty} + \xxi_{i\rightarrow 0}'\|\tilde{\eta}_i\|_{\op}
\leq {\textstyle \frac{1}{1-\alpha}}\big(\|{\eta}_{0}\|_{\infty} + \xxi_{i\rightarrow 0}'\|{\eta}_i\|_{\op}\big),
\end{align*}
and since all $\varepsilon_{0,i}$ have disjoint supports,   $\displaystyle \|\varepsilon_{0}\|_{\infty} \leq {\textstyle \frac{1}{1-\alpha}}\max_{i \in \itgset{r}}\big(\|{\eta}_{0}\|_{\infty} + \xxi_{i\rightarrow 0}'\|{\eta}_i\|_{\op}\big)$. 

On the other hand, 
\begin{align*}
\|\varepsilon_{i}\|_{\op} &= \|\tilde{\eta}_{i} - P_{\T_i}\tilde{\eta}_{0,i}\|_{\op} \\
&\leq \|\tilde{\eta}_{i}\|_{\op} +\| P_{\T_i}\tilde{\eta}_{0,i}\|_{\op}
\leq \|\tilde{\eta}_{i}\|_{\op} + \xxi_{0\rightarrow i}'\|\tilde{\eta}_{0}\|_{\infty}
\leq {\textstyle \frac{1}{1-\alpha}}\big(\|\eta_{i}\|_{\op} + \xxi_{0\rightarrow i}'\|\eta_{0}\|_{\infty}\big).
\end{align*}

Finally,
\begin{align*}
\|\eta_0\|_{\infty} &= \|\P_{\T_0}(u^i{u^i}^{\top}\!)\|_{\infty} \leq \|u^i{u^i}^{\top}\!\|_{\infty} \leq \|u^i\|_{\infty}^2 \leq \frac{\bar{\tau}}{k},\\
\|\eta_i\|_{\op} &= \gamma\|\P_{\T_i}\big (\sign(S^*)\big)\|_{\op} \leq 2\gamma\|\sign(S^*)u^i\|_2 \leq 2 \gamma k_0,
\end{align*}
where we used the fact that $\|\P_{\T_i} (M)\|_{\op}\leq \|M\|_{\op}+\|\P_{\T_i^c} (M)\|_{\op}\leq 2\|M\|_{\op}$ (see the end of the proof of Lemma~\ref{lem:bounds_projq_help}). This concludes the proof.
\end{proof}

\subsection*{Proof of Lemma~\ref{lem:bounds_projq}}
\begin{claim}[Simplified bounds on $\|P_{\T_0^c}q\|_{\infty}$ and $\|P_{\T_i^c}q\|_{\op}$] Let $\alpha:=k_0\sqrt{\frac{2\tauu}{k}}$. If $\alpha<1$, for $q$ as in Lemma~\ref{lem:bounds_projq_help}, we have
\label{lem:bounds_projq}
\begin{align*}
\|P_{\T_0^c}q\|_{\infty} &\leq 
 \frac{\bar{\tau}}{k}\frac{ 1- \alpha + \alpha^2\sqrt{2/k_0}}{1-\alpha}   + \gamma \frac{2 \alpha}{1-\alpha}, \quad
\|P_{\T_i^c}q\|_{\op} &\!\!\!\!\!\leq\gamma k_0 \frac{1+\alpha}{1-\alpha}
+ \frac{\bar{\tau}}{k} \frac{k_0}{1-\alpha}.
\end{align*}

\end{claim}

\begin{proof}
First note that, by Lemma~\ref{lem:bounds_xxi}, we have $\xxi_{0\rightarrow i}\,\xxi_{i\rightarrow 0}\leq \alpha$, so that the results of previous lemmas apply.

We thus start from results of Lemma~\ref{lem:bounds_projq_help}. From definitions, we have
\begin{align*}
\|q_i^*\|_\infty = \|u^i{u^i}^{\top}\|_\infty \leq \max_{k,l} |u^i_k||u^i_l| \leq \frac{\bar{\tau}}{k}.
\end{align*}
and from Lemma~\ref{lem:bounds_xxi}, we have
$\|q_0^*\|_{\op} = \|\gamma \sign(S^*)\|_{\op} \leq \gamma \xxi_{0\rightarrow i} \leq \gamma k_0.$\vspace{2mm}

Then, applying results from Lemma~\ref{lem:bounds_eps}, we get
\begin{align*}
\|P_{\T_0^c}q\|_{\infty} &\leq 
\frac{\bar{\tau}}{k} +  \frac{\xxi_{i\rightarrow 0}}{1-\alpha}\big(2\gamma k_0 + \xxi_{0\rightarrow i}'\frac{\bar{\tau}}{k}\big) \\
&=  \frac{\bar{\tau}}{k}\left(1 + \frac{ \xxi_{0\rightarrow i}'\xxi_{i\rightarrow 0}}{1-\alpha}\right)   + \gamma \frac{2 k_0\xxi_{i\rightarrow 0}}{1-\alpha},\\
\|P_{\T_i^c}q\|_{\op} &\leq  
\gamma k_0 +  \frac{\xxi_{0\rightarrow i}}{1-\alpha}\big(\frac{\bar{\tau}}{k} + \xxi_{i\rightarrow 0}'2\gamma k_0 \big) \\
&= \gamma k_0 \left( 1 + \frac{2\xxi_{0\rightarrow i}\,\xxi_{i\rightarrow 0}'}{1-\alpha}\right)
+ \frac{\bar{\tau}}{k} \frac{\xxi_{0\rightarrow i}}{1-\alpha},
\end{align*}
and then, using agin bounds on $\xxi$ from Lemma~\ref{lem:bounds_xxi}, 
\begin{align*}
\|P_{\T_0^c}q\|_{\infty} &\leq 
 \frac{\bar{\tau}}{k}\left(1 + \frac{ \alpha^2\sqrt{2/k_0}}{1-\alpha}\right)   + \gamma \frac{2 \alpha}{1-\alpha},\\
\|P_{\T_i^c}q\|_{\op} &\leq \gamma k_0 \left( 1 + \frac{2\alpha}{1-\alpha}\right)
+ \frac{\bar{\tau}}{k} \frac{k_0}{1-\alpha}.
\end{align*}
\end{proof}

\subsection*{Proof of Lemma~\ref{eq:valid_gamma_interval}}
\begin{claim}
Let $\alpha:=k_0\sqrt{\frac{2\tauu}{k}}.$ If $\alpha+\frac{\alpha^2}{2k_0} <\frac{1}{3}$, then the interval 
$\Gamma:=\big [ \frac{\bar{\tau}}{k}\frac{ 1}{1-3\alpha}, \frac{1}{k_0} \frac{1-k_0 \tauu/k}{1 + \alpha} \big )$ is not empty,
 and for any $\gamma \in \Gamma$, the dual matrix $q$ defined in Lemma~\ref{lem:bounds_projq_help} satisfies conditions \ref{cond:l1_b} and \ref{cond:om_2}.
\end{claim}

\begin{proof}
Given the inequalities of the previous lemma, a sufficient condition for the inequality $\|P_{\T_0^c}q\|_{\infty} < \gamma$ to hold is if
\begin{align*}
\frac{\bar{\tau}}{k}\frac{ 1- \alpha + \alpha^2\sqrt{2/k_0}}{1-\alpha}   < \gamma \Big(1- \frac{2 \alpha}{1-\alpha}\Big).
\end{align*}
Note that $\alpha\sqrt{\frac{2}{k_0}}\leq \frac{\sqrt{2}}{3}<1.$ As a consequence the previous inequality is implied by the simpler 
\begin{align*}
\frac{\bar{\tau}}{k}\frac{ 1}{1-\alpha}   < \gamma \Big (1- \frac{2 \alpha}{1-\alpha}\Big).
\end{align*}
Clearly, we have $1-3\alpha>0$, so that multiplying the last inequality by $\frac{1-\alpha}{1-3\alpha}$, the last inequality is equivalent to
\begin{align}
\label{eq:gammalb}
\gamma > \frac{\bar{\tau}}{k}\frac{ 1}{1-3\alpha}.   
\end{align}
Similarly, the condition $\big \|P_{\T_i^c}q\big\|_{\op} <1$ is satisfied if
\begin{align}
\gamma k_0 \frac{1+\alpha}{1-\alpha}
 < 1 - \frac{\bar{\tau}}{k} \frac{k_0}{1-\alpha},
\quad \text{or equivalently} \quad
\label{eq:gammaub}
\gamma < \frac{1}{k_0} \frac{1-k_0 \tauu/k}{1 + \alpha}.
\end{align}
Finally combining \eqref{eq:gammaub} and \eqref{eq:gammalb}, we obtain the sufficient condition
$$
\frac{\bar{\tau}}{k}\frac{ 1}{1-3\alpha} \leq \gamma < \frac{1}{k_0} \frac{1-k_0 \tauu/k}{1 + \alpha}.
$$
For $k_0\geq 1,$ this interval is non empty if and only if $2\tauu x_0 (1-\alpha)/(1-3\alpha)<1$, with $x_0:=k_0/k$. But $2\tauu x_0=\frac{\alpha^2}{k_0},$ and $3\alpha+\frac{3\alpha^2}{2k_0}<1$ implies that $(1-3\alpha)^{-1}<\frac{2k_0}{3 \alpha^2}$. So that $2\tauu x_0 (1-\alpha)/(1-3\alpha)\leq \frac{2}{3}(1-\alpha)<1,$ which shows the desired result.
\end{proof}


The final step of the proof of Theorem~\ref{theo:two} is to prove Proposition \ref{prop:end_proof}, which is more involved. The next appendix is devoted to its proof.

\section{Proof of Proposition \ref{prop:end_proof}}
\label{app:proof_prop_two}

Let $m:=|\{i \mid I_i \cap J \neq \varnothing\}|$ denote the number of blocks of the support that are intersecting $J$. Let $k_i:=|I_i \cap J|$. We assume here w.l.o.g.\, that, for the set $J$ we consider, $\{i \mid I_i \cap J \neq \varnothing\}=\itgset{m}$ and that $k_1 \geq k_2\geq \ldots \geq k_m$. In the rest of the proof we will let $x_0:=\frac{k_0}{k}$ and $x_i:=\frac{k_i}{k}$. We will also write $\tilde{I}_i:=(I_1 \cup \ldots \cup I_{i-1})^c.$

\subsection{A recursive decomposition of each submatrix $Q_{\!J\!J}$}
\label{sec:rec_decomp}
We consider a recursive decomposition of this matrix in four blocks
 
$$
Q_{\!J\!J}=\begin{bmatrix}
Q_{J \cap I_1,J \cap I_1} & Q_{J \cap I_1,J \cap I_1^c}\\
Q_{J \cap I_1^c,J \cap I_1} & Q_{J \cap I_1^c,J \cap I_1^c}\\
\end{bmatrix},
$$
then, we redecompose the lower right block as follows 
$$
Q_{J \cap I_1^c,J \cap I_1^c}=\begin{bmatrix}
Q_{J \cap I_2,J \cap I_2} & Q_{J \cap I_2,J \cap (I_1 \cup I_2)^c}\\
Q_{J \cap (I_1 \cup I_2)^c,J \cap I_2} & Q_{J \cap (I_1 \cup I_2)^c,J \cap (I_1 \cup I_2)^c}\\
\end{bmatrix}.
$$
etc, see Figure~\ref{fig:blocks}.

In particular, we will construct upper bounds $\lambdasort{i}$ and $\lambdatildesort{i}$ such that $$\lambda_{\max} (Q_{J \cap I_i,J \cap I_i}) \leq \lambdasort{i} \quad \text{and} \quad\lambda_{\max} (Q_{J \cap \tilde{I}_i,J \cap \tilde{I}_i}) \leq \lambdatildesort{i}.$$

To construct an upper bound of $\lambda_{\max} (Q_{J \cap I_i,J \cap I_i})$ it is necessary to take into account the structure of $Q_{J \cap I_i,J \cap I_i}$ and in particular the fact that, for the operator norm, the component of $Q_{I_i,I_i}$ on $\T_i$ will contribute most strongly to the largest eigenvalue of $Q_{J \cap I_i,J \cap I_i}$, especially when the overlap $J \cap I_i$ is large.

Let $P_{u^i}^{\scriptscriptstyle\bot}:=I-u^i (u^i)^\top$ for short. Note that since $\mathcal{P}_{\mathcal{T}_i^c}(Q)=P_{u^i}^\sbot\, Q_{I_iI_i}\,  P_{u^i}^\sbot$ and $P_{u^i}^{\scriptscriptstyle\bot}$ is idempotent, we have $\mathcal{P}_{\mathcal{T}_i^c}(Q)=P_{u^i}^\sbot \mathcal{P}_{\mathcal{T}_i^c}(Q) P_{u^i}^\sbot.$ 


With these notations and remarks, we have 
\begin{equation}
\label{eq:Q_decomp}
Q_{I_i \cap J, I_i \cap J}=u^i_J{u^i_J}^\top +[P_{u^i}^\sbot \mathcal{P}_{\mathcal{T}_i^c}(Q) P_{u^i}^\sbot]_{\!J\!J}.
\end{equation}
Let $\check{u}^i_J=\frac{u^{i}_J}{\|u^{i}_J\|}$ and $[\check{u}_{J}^i,U_{J}^i]$ be an orthormal basis matrix, obtained from $\check{u}^i_J$ by Gram-Schmidt orthonormalization. Since the matrix $[\check{u}_{J}^i,U_{J}^i]^\top Q_{I_i \cap J, I_i \cap J} [\check{u}_{J}^i,U_{J}^i]^\top,$ has the same largest eigenvalue as $Q_{I_i \cap J, I_i \cap J}$, we consider the four blocks of the former matrix, bound separately the operator norms of each of the blocks and then construct the upper bound $\lambdasort{i}$ from these.

Indeed, using \eqref{eq:Q_decomp} and the fact that $\supp(\check{u}_{J}^i) \subset J,$ we have
\begin{align}
\label{eq:a}
&|(\check{u}_{J}^i)^\top Q_{I_i \cap J, I_i \cap J} \, {\check{u}_{J}^i}| &&\leq \|u^i_J\|_2^2+\|P_{u^i}^\sbot\check{u}_{J}^i\|_2 \|\mathcal{P}_{\mathcal{T}_i^c}(Q)\|_{\op} \|P_{u^i}^\sbot\check{u}_{J}^i\|_2,\\
\label{eq:b}
&\|(U_{J}^i)^\top Q_{I_i \cap J, I_i \cap J} \, {\check{u}_{J}^i}\|_2 &&\leq \|\mathcal{P}_{\mathcal{T}_i^c}(Q)\|_{\op} \|P_{u^i}^\sbot\check{u}_{J}^i\|_2,\\
\label{eq:d}
&\|(U_{J}^i)^\top Q_{I_i \cap J, I_i \cap J} \,U_{J}^i\|_{\op} &&\leq \|\mathcal{P}_{\mathcal{T}_i^c}(Q)\|_{\op}.
\end{align}

We will discuss in the next section how we can leverage these bounds to obtain a bound $\lambdasort{i}$. We first discuss how the various terms appearing in the right hand sides can be bounded based on the assumption and previous results.

As a consequence of the assumed inequalities \eqref{eq:tau_constraints} on $u^i$, we have $x_i \,\taul \leq \|u^i_J\|_2^2 \leq x_i \, \tauu,$
and, using the same formula for $u^i_{J \backslash I_i}$ and combining,
\begin{equation}
\label{eq:ui}
\|u^i_J\|^2 \leq \min (x_i \tauu, 1-\taul+x_i \taul).
\end{equation} 
Note that we have $x_i \tauu < 1-\taul+x_i \taul$ if and only if $2k_i<k$.\vspace{1mm}

We have $\|P_{u^i}^\sbot\check{u}_{J}^i\|_2^2=\|\check{u}_{J}^i-u^i{u^i}^\top\check{u}_{J}^i\|_2^2=1-({u^i}^\top\check{u}_{J}^i)^2=1-\|u_{J}^i\|_2^2,$ so that
\begin{equation}
\label{eq:ui_comp}
\|P_{u^i}^\sbot\check{u}_{J}^i\|_2^2 \leq \min \big (1-\taul x_i,\tauu(1-x_i) \big ).
\end{equation}

Again, which of the two elements in the upper bound is smaller depends on whether $x_i\leq \frac{1}{2}.$

As in the statement of the theorem, we set $\gamma:=\frac{\mu \tauu}{k}$ with $\mu:=(1-3\alpha)^{-1}$ and, as before, $\alpha=k_0\sqrt{\frac{2 \tauu}{k}}$.

Using this value of $\gamma$ in the upper bound obtained in Lemma~\ref{lem:bounds_projq}, we have
\begin{equation}
\label{eq:r}
\|\mathcal{P}_{\mathcal{T}_i^c}(Q)\|_{\op}\leq r:=\frac{\mu \tauu}{k} k_0 \frac{1+\alpha}{1-\alpha}
+ \frac{\bar{\tau}}{k} \frac{k_0}{1-\alpha}=2 \mu \tauu x_0.
\end{equation}

We need to upper bound also the off-diagonal blocks. For this, note that all off-diagonal blocks are in $\mathcal{T}_0$ and that, given that$\|S^*_{i\cdot}\|_0\leq k_0$, for any sets $J',J''$ with $|J'|=k'$ and $|J''|=k''$, it follows from \eqref{eq:cauchy_schwarz_op} that
\begin{equation}
\label{eq:gen_gamma_k0}
\forall Z \in \T_0, \quad  \|Z_{J'\! J''}\|_{\op} \leq \|Z\|_{\infty} \sqrt{\|Z\|_{0}}\leq \|Z\|_{\infty} \sqrt{\min(k',k_0) \min(k'',k_0)}.
\end{equation}

In particular, we have 
\begin{equation}
\label{eq:gamma_k0}
\|Q_{J \cap I_i, J \cap \tilde{I}_{i+1}}\|_{\op} \leq \gamma k\sqrt{\min(x_i,x_0)\min(\tilde{x}_i,x_0)}.
\end{equation}
with $\tilde{x}_i=1-\sum_{j=1}^{i-1}x_j.$
Letting $\check{z}:=\min(x_0,1-x_1)$, this entails
\begin{eqnarray}
\label{eq:corner_bounds}
\|Q_{J \cap I_i, J \cap \tilde{I}_{i+1}}\|_{\op}\leq \gamma k_0
\qquad \text{and} \qquad \|Q_{J \cap I_1, J \cap I^c_{1}}\|_{\op} \leq \gamma \sqrt{k_0k \check{z}}.
\end{eqnarray}

\begin{figure}
\begin{center}
\begin{tikzpicture}[scale=0.50]
\draw (0,0) -- (-8,0) -- (-8,8) -- (-0,8) -- (0,0);
\draw (-4,0) -- (-4,8);
\draw (0,4) -- (-8,4);
\draw (-2,0) -- (-2,4);
\draw (0,2) -- (-4,2);
\draw (-1,0) -- (-1,2);
\draw (0,1) -- (-2,1);
\node at (-3,3) {$\scriptscriptstyle Q_{\!J\!\cap\! I_2\!,\!J\!\cap\! I_2}$};
\node at (-2,6) {$Q_{J\cap I_1^c,J\cap I_1}$};
\node at (-6,2) {$Q_{J\cap I_1,J\cap I_1^c}$};
\node at (-6,6) {$Q_{J\cap I_1,J\cap I_1}$};
\end{tikzpicture} \hspace{1cm}
\begin{tikzpicture}[scale=0.50]
\draw (0,0) -- (-8,0) -- (-8,8) -- (-0,8) -- (0,0);
\draw (-4,0) -- (-4,8);
\draw (0,4) -- (-8,4);
\draw (-2,0) -- (-2,4);
\draw (0,2) -- (-4,2);
\draw (-1,0) -- (-1,2);
\draw (0,1) -- (-2,1);
\node at (-.5,.5) {$\scriptstyle \tilde{\lambda}^4$};
\node at (-.5,1.5) {$\scriptstyle \gamma k_0$};
\node at (-1.5,.5) {$\scriptstyle  \gamma k_0$};
\node at (-1.5,1.5) {$\scriptstyle  \lambda^3$};
\node at (-1,3) {$\gamma k_0$};
\node at (-3,1) {$\gamma k_0$};
\node at (-3,3) {$\lambdasort{2}$};
\node at (-2,6) {$\gamma \sqrt{k_0k\check{z}}$};
\node at (-6,2) {$\gamma \sqrt{k_0k\check{z}}$};
\node at (-6,6) {$\lambdasort{1}$};
\end{tikzpicture}
\vspace{1cm}

\begin{tikzpicture}[scale=0.50]
\draw (0,0) -- (-8,0) -- (-8,8) -- (-0,8) -- (0,0);
\draw (-6,8) -- (-6,4);
\draw (-8,6) -- (-4,6);
\draw (-4,3) -- (-2,3);
\draw (-3,4) -- (-3,2);
\draw (-2,1.5) -- (-1,1.5);
\draw (-1.5,2) -- (-1.5,1);
\draw (-4,0) -- (-4,8);
\draw (0,4) -- (-8,4);
\draw (-2,0) -- (-2,4);
\draw (0,2) -- (-4,2);
\draw (-1,0) -- (-1,2);
\draw (0,1) -- (-2,1);
\node at (-.5,.5) {$\ldots$};
\node at (-.5,1.5) {$\scriptstyle \gamma k_0$};
\node at (-1.5,.5) {$\scriptstyle \gamma k_0$};
\node at (-1.25,1.25) {$\scriptscriptstyle d^3$};
\node at (-1.75,1.25) {$\scriptscriptstyle b^3$};
\node at (-1.25,1.75) {$\scriptscriptstyle c^3$};
\node at (-1.75,1.75) {$\scriptscriptstyle  a^3$};
\node at (-1,3) {$\gamma k_0$};
\node at (-3,1) {$\gamma k_0$};
\node at (-2.5,2.5) {$\scriptstyle d^2$};
\node at (-3.5,2.5) {$\scriptstyle b^2$};
\node at (-2.5,3.5) {$\scriptstyle c^2$};
\node at (-3.5,3.5) {$\scriptstyle a^2$};
\node at (-2,6) {$\gamma \sqrt{k_0k\check{z}}$};
\node at (-6,2) {$\gamma \sqrt{k_0k\check{z}}$};
\node at (-5,5) {$d^1$};
\node at (-7,5) {$b^1$};
\node at (-5,7) {$c^1$};
\node at (-7,7) {$a^1$};
\end{tikzpicture}
\end{center}
\caption{Matrix blocks and corresponding upper bounds on largest singular values:
(top left) Recursive partitioning of blocks of $Q_{JJ}$ introduced in Section~\ref{sec:rec_decomp}
(top right and bottom) Upper bounds on the operator norms of (sub)blocks introduced in inequalities~\eqref{eq:corner_bounds}, \eqref{eq:aibicidi} and in Proposition~\ref{prof:lambdai}.}
\label{fig:blocks}
\end{figure}

\subsection{Bounding eigenvalues of different blocks within $Q_{\!J\!J}$}

To write concisely various bounds we introduce several notations.
First, given a two-by-two matrix $M$ of the form
$$M=\begin{bmatrix}
a & b\\
c & d \\
\end{bmatrix} \qquad \text{with} \quad  a,b,c,d \geq 0,
$$
we denote its largest eigenvalue $\lambda_{\max}(a,b,c,d).$

For $0\leq x\leq 1/2$, with $\etau:=\tauu-r\taul$, and using $r$ defined in \eqref{eq:r},  we denote
$$\begin{cases}
a_l(x) = \tauu x+(1-\taul x) r=\etau x+r\\
b_l(x) = r=c_l(x)\\
d_l(x)= r
\end{cases}
$$
and we write $\lambda_l(x)=\lambda_{\max}(a_l(x),b_l(x),c_l(x),d_l(x)).$ Note that $x \mapsto \lambda_l(x)$ is clearly an increasing function.

Combining inequalities 
\eqref{eq:a} to~\eqref{eq:r}
we get that, for all $i \in \itgset{m}$, if $x_i \in (0,\frac{1}{2}]$, we let
$a^i:=a_l(x_i),b^i:=b_l(x_i),c^i:=c_l(x_i),d^i:=d_l(x_i),$ we have 

\begin{equation}
\label{eq:aibicidi}
\begin{cases}
|(\check{u}_{J}^i)^\top Q_{I_i \cap J, I_i \cap J} \, {\check{u}_{J}^i}|&\leq a^i,\\
\|(U_{J}^i)^\top Q_{I_i \cap J, I_i \cap J} \, {\check{u}_{J}^i}\|_2 &\leq b^i=c^i,\\
\|(U_{J}^i)^\top Q_{I_i \cap J, I_i \cap J} \,U_{J}^i\|_{\op} &\leq d^i.
\end{cases}
\end{equation}
(We could get a smaller value for $b^i=c^i$ based on \eqref{eq:ui_comp}, but this us not useful for our proof)

Symmetrically, for $1>x>1/2$, then if $z=1-x$, and with $\etal:=\taul-r\tauu$, we define 
$$\begin{cases}
a_u(z) = 1-\taul z+\tauu z r=1-\etal z \\
b_u(z) = r \sqrt{\tauu z}=c_u(z)\\
d_u(z)= r.
\end{cases}
$$
We will denote again $\lambda_u(z)=\lambda_{\max}(a_u(z),b_u(z),c_u(z),d_u(z)).$

Combining again inequalities \eqref{eq:a} to~\eqref{eq:r}, we get that, for $x_1 \in [\frac{1}{2},1)$, if
$a^1:=a_l(1-x_1),b^1:=b_l(1-x_1),c^1:=c_l(1-x_1),d^1:=d_l(1-x_1),$ then the set of inequalities~\eqref{eq:aibicidi} holds again. Note that, since by definition $\sum_{i=1}^m x_i \leq 1$, only $x_1$ can possibly be larger than $\frac{1}{2}$.

\begin{proposition}
\label{prof:lambdai}
If for all $i \in \itgset{m}$, we let $\lambdasort{i}:=\lambda_{\max}^+(a^i,b^i,c^i,d^i),$ then 
$$\|Q_{I_i \cap J, I_i \cap J}\|_{\op} \leq \lambdasort{i}.$$
\end{proposition}
\begin{proof}
The result follows from Lemma~\ref{lem:lmax_twobywto} and the fact that, for $a,b,c,d\geq 0$, if we have $a'\!\geq\! a, b'\!\geq\! b, c'\!\geq\! c,d'\!\geq\! d,$ then $\lambda^+_{\max}(a,b,c,d) \leq \lambda^+_{\max}(a',b',c',d')$.
\end{proof}

\begin{proposition}
\label{prop:lambda_tilde}
For $i\geq 2$, 
$\quad\|Q_{\tilde{I}_i \cap J, \tilde{I}_i \cap J}\|_{\op} \leq \lambdatildesort{i}:=\lambdasort{i}+\gamma k_0.$
\end{proposition}
\begin{proof}
To keep notations as simple as possible we prove the result for $\lambdatildesort{2}$. The proof is the same for larger values of $i$.
\begin{eqnarray*}
\|Q_{J \cap I_1^c,J \cap I_1^c}\|_{\op}&\leq& \max_{2 \leq i \leq m}\|Q_{J \cap I_i,J \cap I_i}\|_{\op}+\|Q_{(J \cap I_1^c \times J \cap I_1^c)\backslash\bigcup_{2 \leq i \leq m} I_{i} \times I_{i}}\|_{\op}\\
&\leq &\max_{2 \leq i \leq m} \lambdasort{i} +\gamma k_0
\leq  \lambdasort{2} +\gamma k_0,
\end{eqnarray*}
where the second inequality is a variant of \eqref{eq:corner_bounds} due to \eqref{eq:gen_gamma_k0}, and because we have $\lambdasort{2}\geq \ldots \geq \lambdasort{m}$ given that $z \mapsto \lambda_l(z)$ is non-decreasing.
\end{proof}

Note that by Lemma~\ref{lem:eig_ub}, we have $\lambda_l(x) \leq a_l(x)+\sqrt{b_l(x)c_l(x)}\leq \tauu x+ 2r$ and
 $\lambda_u(z) \leq 1-\etal z+ r,$
since $r<1-\etal z$ for $z \leq \frac{1}{2}.$

\subsection{Some technical lemmas to quantify eigenvalue bounds}

We first derive a bound applicable to $\lambdasort{1}$ if $x_1 \leq \frac{1}{2}$ and to all $(\lambdasort{i})_{2 \leq i \leq m}$ since, for $i \geq 2$, $0<x_i\leq \frac{1}{2}.$
First note that, $k_0\leq \frac{1}{7} \sqrt{k}$ entails that $\mu\leq 7$, for $C\geq182$, we have $x_0 \leq \frac{1}{182}$, and so $r \leq 2 \mu \tauu x_0 \leq \frac{1}{4}$ and $\gamma k_0 \leq \mu \tauu x_0 \leq \frac{1}{8}$.

\begin{lemma} For $0<x\leq \frac{1}{2}$, we have $\displaystyle \lambda_l(x)< 1-2\gamma k_0.$
\end{lemma}
\begin{proof}
We show that $(1-2\gamma k_0-a_l(x))(1-2\gamma k_0-d_l(x)) \geq b_l(x)c_l(x)$.
In the calculation, we will write $\etau=\tauu-\taul r$ for short.

We have $r \leq \frac{1}{4}$ and $\gamma k_0 \leq \frac{1}{8}$, which, given that $\tauu \geq 1$, entails that 
$\tauu -2r-2\etau \gamma k_0+r^2\taul \geq \tauu(1-2\gamma k_0)-2r \geq \frac{1}{4}>0.$

As a consequence,
\begin{eqnarray*}
& & (1-2\gamma k_0-a_l(x))(1-2\gamma k_0-d_l(x))- b_l(x)c_l(x)\\
&=& (1-2\gamma k_0-\etau x -r)(1-2\gamma k_0-r)- r^2\\
&=& (1-r-2\gamma k_0)^2-\etau x(1-r-2\gamma k_0)- r^2 \\
&=& (1-r-2\gamma k_0)^2-(\tauu-\taul r) x(1-r)- r^2+2\etau \gamma k_0 x\\
&=& (1-r-2\gamma k_0)^2-(\tauu-\taul r-\tauu r+\taul r^2)x- r^2+2\etau \gamma k_0 x\\
&=& (1-2r)-4(1-r)\gamma k_0+4\gamma^2 k_0^2-(\tauu -2r-2\etau \gamma k_0+r^2\taul) x\\
&\geq& (1-2r)-4(1-r)\gamma k_0-{\textstyle \frac{1}{2}}\tauu +r+\etau \gamma k_0-{\textstyle \frac{1}{2}} r^2\taul\\
&\geq & {\textstyle \frac{1}{2}}\taul (1-4\gamma k_0-r^2)-r
\geq  {\textstyle \frac{1}{2}}\tauu x_0 ( \kappa\tauu({\textstyle 1-\frac{1}{2}-\frac{1}{16}}) -4 \mu) > 0,
\end{eqnarray*}
The last equality and the second inequality use $\tauu+\taul=2$, the first inequality uses that the expression is a decreasing function of $x$ on $(0,\frac{1}{2}]$, the penultimate inequality uses that, by assumption, the inequalities~\eqref{eq:tau_constraints} hold, and in particular $\taul \geq \kappa \tauu^2 x_0$
 and again that $r \leq \frac{1}{4}$ and $\gamma k_0 \leq \frac{1}{8}$, and, the final positivity stems from the assumption, made in the statement of the theorem, that $\kappa> 16 \mu$. 
\end{proof}

\begin{corollary}
\label{cor:lower_block_ub}
We have, for all $i \geq 2$,
$\displaystyle \lambdatildesort{i}\leq 1-\gamma k_0.$
\end{corollary}
\begin{proof}
Immediate from the previous result since $\lambdatildesort{i}\leq \lambdasort{i} +\gamma k_0$ 
\end{proof}
We now upper bound $\lambdasort{1}$ in the case where $x_1>\frac{1}{2}$. Indeed, in that case we have $\lambdasort{1}=\lambda_u(1-x_1)$ and the bound is provided by the following result.
\begin{lemma}
\label{lem:etaxi}
For $0<z<\frac{1}{2}$, $$\lambda_u(z)<1-(\etal-\xi)z \qquad \text{with} \quad \etal:=\taul-\tauu r \quad  \text{and} \quad \xi:=2 \tauu r^2/(\tauu-2r).$$
\end{lemma}
\begin{proof}
First note that 
\begin{eqnarray*}
\etal-\xi=\taul-\tauu r-\frac{2 \tauu r^2}{\tauu-2 r}=\frac{\taul \tauu-\tauu^2 r-2\taul r}{\tauu-2 r}
&\geq& \frac{\tauu}{\tauu-2 r}(\taul-\tauu r-2 r) \\
&\geq&(\kappa-6 \mu)\tauu x_0>0.
\end{eqnarray*}

We clearly have $(1-(\etal-\xi)z \big )-a_u(z)=\xi z>0$ and $(1-(\etal-\xi)z \big )-d_u(z)\geq \frac{1}{2}-r>0$.
Moreover, we have
\begin{eqnarray*}
& & \big ((1-(\etal-\xi)z \big )-a_u(z)) \big  ((1-(\etal-\xi)z)-d_u(z) \big  )-b_u(z)c_u(z)\\
 &=&\xi z\big (1-(\etal-\xi)z-r \big )-\tauu r^2 z=\xi z(1-r)-\xi z^2(\etal-\xi)-\tauu r^2 z>0
\end{eqnarray*}
because
$$\xi (1-r)-\xi z(\etal-\xi)-\tauu r^2>\xi (1-r)-\xi {\textstyle \frac{\taul}{2}}-\tauu r^2\geq \xi (1-r)-\xi {\textstyle \frac{\taul}{2}}-({\textstyle \frac{\tauu}{2}}-r)\xi=0,$$
where the first strict inequality is obtained using $\etal-\xi>\taul$ and the fact that, given that $\etal-\xi>0$, $z=\frac{1}{2}$ must minimize the expression.
This proves the result by application of Lemma~\ref{lem:cns_lmaxleqone}.
\end{proof}

\subsection{Combining bounds on eigenvalues of sublocks of $Q_{\!J\!J}$}
We can finally prove the claim of Proposition~\ref{prop:end_proof}:
\begin{claim}
Under the assumptions of Theorem~~\ref{theo:two}, and setting $C$ in the statement of the theorem to $C=182$, then for all $J \in \mathcal{G}^p_k \backslash\{I_i\}_{1\leq i \leq m},$
we have $\lambda_{\max}^+(Q_{\!J\!J})<1.$
\end{claim}
\begin{proof}
Note first that, as discussed at the beginning of the previous section, under these assumptions, we have  $r \leq \frac{1}{4}$ and $\gamma k_0 \leq \frac{1}{8}$.

To prove the result, we distinguish four cases:

\textbf{1st case: $0 \leq x_1 \leq \frac{1}{2}$.} If $0\leq x_1\leq \frac{1}{2}$, then by the same argument as in Corollary~\ref{cor:lower_block_ub}, we have
$$\lambda_{\max}^+(Q_{\!J\!J})\leq \lambdatildesort{1}\leq (1-\gamma k_0)< 1.$$

\textbf{2nd case: $\frac{1}{4} \leq z:=1-x_1 \leq \frac{1}{2}$.}
 
If $x_1 > \frac{1}{2}$, then we let $z=1-x_1$, and we can upper bound the largest eigenvalue of the upper left block in Figure~\ref{fig:blocks} by $\lambdasort{1}=\lambda_u(z)$ and the lower right block by $\lambdatildesort{2}=\tilde{\lambda}_l(z):=\lambda_l(z)+\gamma k_0$, given Proposition~\ref{prop:lambda_tilde}.

First, we consider the case $z:=1-x_1$ with $\frac{1}{4} \leq z \leq \frac{1}{2}$.
In that case, we have 
$\lambda_{\max}^+(Q_{\!J\!J})\leq \lambda_{\max}^+(\lambdasort{1},\gamma k_0,\gamma 
k_0,\lambdatildesort{2}).$
But by Lemma~\ref{lem:etaxi} and Corollary~\ref{cor:lower_block_ub}, we have
\begin{eqnarray*}
\big (1-\lambda_u(z) \big  )\big  (1-\tilde{\lambda}_l(z) \big  )-\gamma^2k_0^2 & \geq & (\etal-\xi) \frac{1}{4} \gamma k_0-\gamma^2k_0^2,
\end{eqnarray*}
and $\etal-\xi-4\gamma k_0>(\kappa-6 \mu-4\mu)\tauu x_0>0,$ using the same lower bound for 
$\etal-\xi$ as the one established in Lemma~\ref{lem:etaxi}.

\textbf{3rd case: $x_0 \leq z=1-x_1 \leq \frac{1}{4}$.}
We have $$\lambda_u(z)\leq a_u(z)+\frac{b_u(z)c_u(z)}{a_u(z)-d_u(z)}=1-\etal z+\frac{r^2\tauu z}{1-\etal z-r}\quad  \text{and}  \quad \tilde{\lambda}_l(z) \leq \tauu z+2r+\gamma k_0.$$ 
As a consequence the function $f$ defined by
$$f(z):=\Big (\etal z-\frac{r^2\tauu z}{1-\etal z-r}\Big )\big (1-\tauu z-2r-\gamma k_0\big  )$$
provides the lower bound $f(z)\leq \big (1-\lambda_u(z) \big  )\big  (1-\tilde{\lambda}_l(z) \big  ).$

We first show that this function is increasing on the interval $[x_0,\frac{1}{4}].$
Indeed, given that, for $z\leq \frac{1}{4}$, we have $\etal z +r \leq \frac{1}{2}$, we have
\begin{eqnarray*}
f'(z)&=&\Big ( \etal-\frac{r^2\tauu}{1-\etal z-r} -\frac{\etal r^2 \tauu z}{(1-\etal z-r)^2} \Big ) (1-\tauu z-2r-\gamma k_0)-\etal z \tauu + \frac{r^2 \tauu^2 z}{1-\etal z -r}\\
&=&\Big ( \etal-\frac{r^2\tauu}{1-\etal z-r} -\frac{\etal r^2 \tauu z}{(1-\etal z-r)^2} \Big ) (1-2\tauu z-2r-\gamma k_0)  -\frac{\etal r^2 \tauu^2 z^2}{(1-\etal z -r)^2}\\
&\geq & \big ( \etal -2 r^2 \tauu -\etal r^2 \tauu\big ) \big ( \frac{\taul}{2}-2r-\gamma k_0 \big )-\frac{1}{4}\etal r^2 \tauu^2\\ 
&\geq& \big( (\kappa-2 \mu)- \mu-\frac{1}{2} \mu \big ) 
\frac{1}{2}\big( \kappa-8\mu-\mu \big )\tauu^3 x_0^2-\mu^2\tauu^4x_0^2\\
&\geq& \frac{\tauu^3x_0^2}{2} \big [ (\kappa-4\mu) (\kappa-9 \mu)-4\mu^2\big ] >0.
\end{eqnarray*}
Therefore the minimal value of $f$ is attained for $z=x_0$. Note that
$$\etal (1-\etal x_0 -r)-r^2 \tauu=\taul-r\tauu-\etal^2x_0-r \taul=\taul-2r-\etal^2 x_0 \geq \taul (1-x_0)-2\tauu r$$
which entails
\begin{eqnarray*}
(1-\etal x_0 -r) \big [f(x_0)-\gamma^2 k_0^2]
&\geq & (1-\etal x_0 -r) f(x_0)-\gamma^2 k_0^2\\
&\geq & x_0 (\taul(1-x_0)-2\tauu r)(1-\tauu x_0-2r-\mu \tauu x_0)-\mu^2 \tauu^2 x_0^2\\
&\geq &  \tauu^2 x^2_0 \big [(\kappa (1-x_0)-4 \mu){\textstyle \frac{3}{4}}-\mu^2 \big] \geq 4 \, \tauu^2 x^2_0 \, \mu>0,
\end{eqnarray*}
since $\mu\leq 7$ and since, the assumption $x_0\leq \frac{1}{182}$ entails that $\tauu x_0+2r+\mu \tauu x_0 \leq \frac{1}{4}$.

But this shows that $f(x_0)-\gamma^2 k_0^2>0$ and since this is a lower bound on 
$$\big (1-\lambda_u(z) \big  )\big  (1-\tilde{\lambda}_l(z) \big  )-\gamma^2k_0^2$$
on the interval $x_0 \leq z \leq \frac{1}{4}$ we again have that $\lambda_{\max}^+(Q_{\!J\!J})<1$ by Lemma~\ref{lem:cns_lmaxleqone}.

\textbf{4th case: $0<z=1-x_1\leq x_0.$}
When $z$ becomes very small, the off-diagonal block $Q_{J \cap I_1,J \cap I_1^c}$ becomes a very thin vertical block. As a consequence the bound $\|Q_{J \cap I_1,J \cap I_1^c}\|_{\op} \leq \gamma k_0$ is no longer sufficient, but using Equation~\eqref{eq:gen_gamma_k0} we also have that $\|Q_{J \cap I_1,J \cap I_1^c}\|_{\op} \leq \tilde{b}(z)$ with $\tilde{b}(z)=\gamma \sqrt{k_0kz}$. As a consequence, we have
$$\lambda_{\max}^+(Q_{\!J\!J})\leq \lambda_{\max}^+(\lambdasort{1},\tilde{b}(z),\tilde{b}(z),\lambdatildesort{2}),$$
with $\lambdasort{1}=\lambda_u(z)$, $\lambdatildesort{2}=\tilde{\lambda}_l(z)$ and $z=1-x_1$.
Reasoning like for the 3rd case, since $f(z)-\gamma^2 k_0 k z\leq \big (1-\lambda_u(z) \big  )\big  (1-\tilde{\lambda}_l(z) \big  )-\tilde{b}(z)^2$, it is sufficient to prove that $f(z)-\gamma^2 k_0 k z>0.$ But since $0<z\leq x_0$, we simply have
\begin{eqnarray*}
\frac{x_0}{z}(f(z)-\gamma^2 k_0 k z) &=& x_0 \Big (\etal-\frac{r^2\tauu}{1-\etal z-r}\Big )\big (1-\tauu z-2r-\gamma k_0\big  ) -\gamma^2 k_0^2\\
&\geq & x_0 \Big (\etal-\frac{r^2\tauu}{1-\etal x_0-r}\Big )\big (1-\tauu x_0-2r-\gamma k_0\big ) -\gamma^2 k_0^2\\
& = & f(x_0)-\gamma^2 k_0^2>0,
\end{eqnarray*}
where the last inequality was proven in the analysis of the 3rd case. This shows that for all $0<z\leq x_0$, we have  $$0<f(z)-\gamma^2 k_0 k z\leq \big (1-\lambda_u(z) \big  )\big  (1-\tilde{\lambda}_l(z) \big  )-\tilde{b}(z)^2,$$
so that $\lambda_{\max}^+(Q_{\!J\!J})<1$ by Lemma~\ref{lem:cns_lmaxleqone}.
\end{proof}

\section{Lemmas to control eigenvalues}
In this section, we establish general bounds on eigenvalues of two-by-two matrices and of matrices that can be partitioned in  two-by-two blocks. 

Consider a two-by-two matrix $M$ of the form
$$M=\begin{bmatrix}
a & b\\
c & d \\
\end{bmatrix} \qquad \text{with} \quad  a,b,c,d \geq 0.
$$
We denote its largest eigenvalue $\lambda_{\max}.$

Since $\lambda_{\max}+\lambda_{\min}=a+d$ and $\lambda_{\max}\lambda_{\min}=ad-bc$, the eigenvalues are the roots of $x^2-(a+d)x+ad-bc$, and by the quadratic formula, we have
$$2\lambda_{\max} = a+d+\sqrt{(a-d)^2+4bc}.$$
Given that $a,b,c,d \geq 0$, we must have $(a-d)^2+4bc>0$ and the eigenvalues of $M$ are real.
\begin{lemma}
\label{lem:cns_lmaxleqone}
$\displaystyle \lambda_{\max}<\nu \quad\Leftrightarrow \quad \max(a,d)<\nu \quad \text{and} \quad bc < (\nu -a)(\nu-d).$
\end{lemma}
\begin{proof}
Indeed we clearly have $\lambda_{\max}<\nu \Rightarrow \max(a,d)<\nu.$
And conversely, if $\max(a,d)<\nu$, using the quadratic formula, we have
\begin{eqnarray*}
\lambda_{\max}<\nu &\Leftrightarrow& a+d+\sqrt{(a-d)^2+4bc}<2\nu\\
&\Leftrightarrow& (a-d)^2+4bc<(2\nu-(a+d))^2\\
&\Leftrightarrow& -2ad+4bc<4\nu^2-4\nu (a+d)+2ad\\
&\Leftrightarrow&  bc<\nu^2-2 (a+d)+ad.
\end{eqnarray*}
where the second equivalence uses that $\max(a,d)<\nu \:\Rightarrow\: 2\nu-a-d>0.$
\end{proof}

\begin{lemma}
\label{lem:eig_ub}
If $a>d$, we have  $\displaystyle \lambda_{\max}\leq a+\frac{bc}{a-d}$ and $\displaystyle \lambda_{\max}\leq a+\sqrt{bc}.$
\end{lemma}
\begin{proof}
Indeed, if $a>d,$
$$\sqrt{(a-d)^2+4bc}  \leq (a-d) \sqrt{1+\frac{4bc}{(a-d)^2}} \leq (a-d) \Big (1+\frac{2bc}{(a-d)^2} \Big ) \leq a-d +\frac{2bc}{a-d}.$$
So that by the quadratic formula, we have
$$2 \lambda_{\max} = a+d+\sqrt{(a-d)^2+4bc}  \leq a+d+a-d +\frac{2bc}{a-d}
= 2a +\frac{2bc}{a-d}.$$
To prove the second inequality, note that 
$\sqrt{(a-d)^2+4bc} \leq a-d+2 \sqrt{bc}$ which yields the result.

\end{proof}

\begin{lemma}
\label{lem:lmax_twobywto}
$$
\lambda_{\max} \left (
\begin{bmatrix}
A & B \\
C & D
\end{bmatrix}
\right)
 \leq 
\lambda_{\max} \left (
\begin{bmatrix}
\lambda_{\max}(A) & \|B\|_{\op} \\
\|C\|_{\op} & \lambda_{\max}(D)
\end{bmatrix}
\right ) 
$$
\end{lemma}
\begin{proof}
Since, for $y_1=\|x_1\|$ and $y_2=\|x_2\|$, we have
$$x_1^\top A x_1 + x_1^\top B x_2 + x_2^\top C x_1 + x_2^\top D x_2 \leq  \lambda_{\max}(A) \, y_1^2 +  (\|B\|_{\op}+  \|C\|_{\op}) \, y_1 y_2+ \lambda_{\max}(D) \, y_2^2,$$
 maximizing on both sizes of the inequality under the constraint $y_1^2+y_2^2=1$ yields the result. 
\end{proof}

\section{Construction of sparse precision matrices}
\label{app:sparse_wishart}
In this appendix, we provide details on the construction of the precision matrices used in the experiments.

Constructing valid concentration matrices for a sparse Gaussian graphical model associated with a given graph is not completely immediate. In our synthetic experiment, we generate random concentration matrices from a model that yields sparse counterparts to Wishart matrices.

Given a graph ${G}=(V,E)$, where $V$ and $E$ are the set of vertices and edges respectively, we first build an incidence matrix $B\in\RR^{n\times m}$ for $G$  (where $n=|V|$ and $m=|E|$, and with $B_{i,j} = 1$ if the vertex $v_i$ and edge $e_j$ are incident and $0$ otherwise). We then compute  a sparse random matrix $\tilde{B}$ with sparsity pattern given by $B$, and with its nonzero coefficients drawn i.i.d. standard Gaussian. Finally, the matrix $K=\tilde{B}\tilde{B}^{\top}$ is a random concentration matrix  with the imposed sparse structure: indeed, by construction, the non-zero pattern of $K$ matches exactly the adjacency structure $E$ of the graph $G$, and the obtained matrix $K$ is clearly p.s.d. \ .

\end{document}